\documentclass{article}

\usepackage{microtype}
\usepackage{graphicx}
\usepackage{subcaption}
\usepackage{booktabs} 

\usepackage{hyperref}

\usepackage[preprint]{icml2026}

\usepackage{amsmath}
\usepackage{amssymb}
\usepackage{mathtools}
\usepackage{amsthm}

\usepackage[capitalize,noabbrev]{cleveref}

\theoremstyle{plain}
\newtheorem{theorem}{Theorem}[section]
\newtheorem{proposition}[theorem]{Proposition}

\theoremstyle{definition}

\theoremstyle{remark}

\usepackage[textsize=tiny]{todonotes}

\usepackage{xcolor}         
\definecolor{mycyan}{RGB}{0,204,204}
\definecolor{gold}{RGB}{204,102,20}
\definecolor{mygreen}{RGB}{47, 166, 36}

\DeclareMathOperator*{\argmin}{arg\,min}

\usepackage{placeins}

\icmltitlerunning{Learning Policy Representations for Steerable Behavior Synthesis}

\begin{document}

\twocolumn[
  \icmltitle{Learning Policy Representations for Steerable Behavior Synthesis}



  \icmlsetsymbol{equal}{*}

  \begin{icmlauthorlist}
    \icmlauthor{Beiming Li}{penn}
    \icmlauthor{Sergio Rozada}{urjc}
    \icmlauthor{Alejandro Ribeiro}{penn}
  \end{icmlauthorlist}

  \icmlaffiliation{penn}{Department of Electrical and Systems Engineering, University of Pennsylvania, Philadelphia, USA}
  \icmlaffiliation{urjc}{Department of Signal Theory and Communications, King Juan Carlos University, Madrid, Spain}

  \icmlcorrespondingauthor{Beiming Li}{beimingl@seas.upenn.edu}

  \icmlkeywords{Policy Representation, Reinforcement Learning, Contrastive Learning, Behavior Synthesis}

  \vskip 0.3in
]



\printAffiliationsAndNotice{}  

\begin{abstract}

Given a Markov decision process (MDP), we seek to learn representations for a range of policies to facilitate behavior steering at test time. As policies of an MDP are uniquely determined by their occupancy measures, we propose modeling policy representations as expectations of state-action feature maps with respect to occupancy measures.  We show that these representations can be approximated uniformly for a range of policies using a set-based  architecture. Our model encodes a set of state-action samples into a latent embedding, from which we decode both the policy and its value functions corresponding to multiple rewards. We use variational generative approach to induce a smooth latent space, and further shape it with contrastive learning so that latent distances align with differences in value functions. This geometry permits gradient-based optimization directly in the latent space. Leveraging this capability, we solve a novel behavior synthesis task, where policies are steered to satisfy previously unseen value function constraints without additional training.

\end{abstract}

\section{Introduction}
\label{sec::intro}

Policies are central to Markov decision processes (MDPs), as they define the behavior of the agent \cite{puterman1990markov}. In many settings, it is desirable to compare, share, or reuse policies across tasks or agents. However, policies are functions over state spaces, which makes them difficult to manipulate directly.
Policy representations address this challenge by mapping policies to fixed-dimensional vectors that capture essential behavioral properties. This shift from functions to vectors has proven effective in domains such as multi-agent games for opponent modeling~\cite{hong2017deep, grover2018learning} and meta-learning for policy transfer across tasks~\cite{Hausman2018LearningAE, rakelly2019efficient}.

Despite these successes, learning general policy representations remains a challenge, as existing approaches are often tied to specific downstream applications rather than modeling general policy behavior. For example, in meta-learning, latent variables are primarily used for task inference \cite{Hausman2018LearningAE, rakelly2019efficient}. Furthermore, existing optimization methods tend to produce unstructured latent spaces. 
Works relying on contrastive objectives \cite{hoffer2015deep} favor discrimination over continuity, leading to fragmented representations \cite{grover2018learning, zha2023rank}. 
In contrast, variational generative approaches can promote smooth latent spaces but fail to align latent geometry with behavioral attributes, so proximity in the latent space does not imply policy similarity \cite{wang2017robust}. 
Consequently, existing policy representations have limited practical utility, especially for methods that require optimization directly in representation space \cite{lecun2022path}.

To address these limitations, we propose a general policy representation model that encodes policies as expectations of state–action features under their occupancy measures, extending successor features \cite{barreto2017successor} to a principled characterization of policy behavior.
We derive a practical estimator based on sampled state-action pairs and show that set-based architectures can approximate the ideal representation up to sampling error. 
Building on this formulation, we learn a structured latent space that supports policy reconstruction while explicitly aligning latent distances with value function differences, yielding a representation amenable to gradient-based search.
We exploit this structure to enable a novel behavior synthesis task, where new policies satisfying prescribed value-function constraints are obtained by optimization in latent space at test time, without retraining or extra environmental interactions. Our contributions are
\begin{itemize}
    \item[\textbf{C1}] We introduce a  policy representation model based on expected feature maps, together with a practical estimator from sampled state–action pairs.
    \item[\textbf{C2}] We learn policy-discriminative representations endowed with a value-function-aligned geometry through a combination of generative and contrastive learning.
    \item[\textbf{C3}] We enable test-time behavior synthesis by navigating the learned policy representation space through gradient-based optimization.
\end{itemize}

\section{Policy Representations}
\label{sec::model}

This section introduces the proposed policy representation model.
Section~\ref{sec::prelims} presents preliminaries,
Section~\ref{sec::repr_model} introduces the model and its practical estimator, and 
Section~\ref{sec::repr_arch} describes the resulting architecture.

\subsection{Preliminaries}
\label{sec::prelims}
We consider a discounted multi-objective MDP defined by the tuple $(\mathcal{S}, \mathcal{A}, \mathcal{P}, r, \gamma, \mu_0)$. Here, $\mathcal{S} \subset \mathbb{R}^{d_s}$ and $\mathcal{A} \subset \mathbb{R}^{d_a}$ denote the bounded state and action spaces, $\mathcal{P}(\cdot|s,a)$ is the transition kernel, $r: \mathcal{S} \times \mathcal{A} \to \mathbb{R}^K$ is the vector-valued reward function, $\gamma \in [0,1)$ is the discount factor, and $\mu_0$ is the initial state distribution. Each policy $\pi$ induces a unique discounted state--action occupancy measure $d_\pi(s,a) = (1-\gamma)\sum_{t=0}^{\infty} \gamma^t \mathbb{P}_\pi(s_t = s, a_t = a)$, where $\mathbb{P}_\pi$ denotes the probability law induced by $\pi$ and the transition dynamics. The discounted return for the $k$-th objective is a random variable $R^{(k)} = \sum_{t=0}^{\infty} \gamma^t r^{(k)}(s_t,a_t)$, whose realizations arise from trajectories $\{(s_t,a_t)\}_{t\ge0}$ collected by executing $\pi$. In this work, we define value function as the expected cumulative reward, $v_\pi^{(k)} = \mathbb{E}_{\pi}\!\left[\sum_{t=0}^{\infty} \gamma^t r^{(k)}(s_t,a_t)\right]$.
For compactness, we define the return vector $R=[R^{(1)},\ldots,R^{(K)}]^\top$ and the value vector $v_\pi=[v^{(1)}_\pi,\ldots,v^{(K)}_\pi]^\top$.

\subsection{Policy representation model}\label{sec::repr_model}

We aim to represent policies as fixed-dimensional vectors. 
Since, under fixed dynamics and initial distribution, the occupancy measure $d_\pi$ uniquely determines the policy $\pi$ up to measure-zero sets, it provides a natural basis for constructing policy representations.
Therefore, we propose mapping each policy $\pi$ to the expected value of an state--action feature map $f$ under its state--action occupancy measure $d_\pi$
\begin{equation}
    h_\pi
    :=
    \mathbb{E}_{ d_\pi}\left[f(s,a)\right] \in \mathcal{H},
    \label{eq:hpi_def}
\end{equation}
where $\mathcal{H}$ denotes the policy representation space.
The feature map $f$ embeds state--action pairs into a latent vector, while $d_\pi$ aggregates these embeddings. 
Since the dynamics are fixed, variations in the resulting representation arise solely from changes in the policy.
This model extends the successor-feature representation framework \cite{barreto2017successor} from state-action pairs to policies, yielding reward-agnostic representations as well. The feature map $f$ may vary depending on application, and different choices induce latent spaces $\mathcal H$ with different properties. 
In this work, however, we do not fix $f$ explicitly; instead, we characterize the representation by enforcing general desirable properties directly on $\mathcal H$, as detailed in later sections. 
For theoretical analysis, we assume that $f$ is well structured in the sense that it belongs to a reproducing kernel Hilbert space (RKHS) associated with a smooth kernel $\kappa$. 
This assumption ensures that nearby state--action pairs admit similar feature representations, without committing to a specific choice of $f$.

The representation in \eqref{eq:hpi_def} cannot be computed since we usually do not have access to $d_\pi$, but only to finite samples generated by executing the policy $\pi$. 
Therefore,  \eqref{eq:hpi_def} must be approximated from a dataset $\mathcal{D}_\pi=\{(s_n,a_n)\}_{n=1}^N$ of state--action pairs, that are assumed to be i.i.d. 
Since $f$ is assumed to lie in an RKHS, a natural estimator of $h_\pi$ is provided by the control functional construction of \citet[Lemma~3]{oates2017control}, which can be summarized as
\begin{equation}
    g(\mathcal{D}_\pi)
    =
    \sum_{n=1}^N w(s_n, a_n \mid \mathcal{D}_\pi)\, f(s_n,a_n).
    \label{eq:hpi_quad}
\end{equation}
The weights $w(s_n, a_n \mid \mathcal{D}_\pi)\in\mathbb{R}$ depend on kernel-based similarities among the sampled state--action pairs (see Appendix \ref{app:control-functional} for details). Under suitable regularity conditions, the resulting sample-based estimator achieves an error rate of order $O(N^{-1})$ \citep[Theorem 2]{oates2017control}, improving upon the standard Monte Carlo rate $O(N^{-1/2})$.

In practice, neither these weights $w$ nor the feature map $f$ are directly available and must therefore be approximated.
We address this limitation by learning the parametric model
\begin{equation}
    \tilde{h}_\pi := g_\theta (\mathcal{D}_\pi)
    =
    \sum_{n=1}^N w_{\theta_0}(s_n,a_n | \mathcal{D}_\pi)\, f_{\theta_1}(s_n,a_n),
    \label{eq:hpi_learned}
\end{equation}
where $\theta:=\{\theta_0,\theta_1\}$ and $\tilde h_\pi \in \mathcal{H}_\theta$. 
The estimator $w$ is a permutation-invariant function that maps a dataset $\mathcal{D}_\pi$ to a scalar weight for a fixed state–action pair $(s,a)$.
Therefore, it is natural to approximate it using a set-based architecture, e.g., DeepSets \cite{zaheer2017deep}.
In contrast, the feature map $f$ acts pointwise on  state–action pairs and is naturally parameterized by a standard neural network, e.g., multilayer perceptron (MLP).
Since DeepSets are universal approximators of permutation-invariant set functions \cite{zaheer2017deep} and MLPs are universal approximators of pointwise functions \cite{cybenko1989approximation}, this parameterization allows the learned estimator to approximate the ideal sample-based estimator $g(\mathcal{D_\pi})$ arbitrarily well. 
This is formalized in the following Theorem, with proof given in Appendix~\ref{app::thm_existence}.

\begin{theorem}
\label{thm:existence_vs_tilde}
Assume: (i) $\Omega=\mathcal S\times\mathcal A$ is compact;
(ii) the kernel $\kappa$ is bounded and twice continuously differentiable; and
(iii) for all $\pi \in \Pi$, the  occupancy measure $d_\pi$ admits a continuously differentiable density on $\Omega$ which vanishes on $\partial\Omega$.
Then, for any $\varepsilon>0$, there exist parameters $\theta_0$ and $\theta_1$
such that the following holds for all $\pi\in\Pi$

\begin{equation}
\label{eq:existence_vs_tilde}
    \|\tilde h_\pi - h_\pi\|
    \le
    C N^{-1}
    + \varepsilon.
\end{equation}
\end{theorem}

In a nutshell, Theorem~\ref{thm:existence_vs_tilde} shows that set-based models can approximate the policy representation in~\eqref{eq:hpi_def} up to a sampling error proportional to $N^{-1}$, provided the feature map $f$ satisfies mild regularity conditions.
Since policies are represented through samples, this guarantee holds uniformly over a class of well-behaved policies. 
The total error decomposes into a sampling term, controlled by the control functional construction of \citet{oates2017control}, and an approximation term $\varepsilon$ arising from learning both the weights and the feature map.
The proof in Appendix~\ref{app::thm_existence} provides a simple construction of $w_{\theta_0}$ and $f_{\theta_0}$ using Deep Sets and MLPs, respectively.
While this construction is sufficient for establishing existence, it may be overly restrictive in practical settings. 
More generally, the representation in~\eqref{eq:hpi_learned} can be interpreted as a set-based attention pooling operator, where the pointwise map $f_{\theta_1}$ acts as the value projection for attention, and the invariant model $w_{\theta_0}(\cdot \mid \mathcal D_\pi)$ defines data-dependent attention weights. 
This perspective naturally accommodates more expressive attention-based architectures as refinements of the underlying weighted-sum construction.

\subsection{Architecture}
\label{sec::repr_arch}

So far, we have introduced a model that embeds a policy $\pi$ into a representation $ h_\pi$.
To ensure the representation is useful across tasks, it must retain information about the behavior of the agent.
Such behavior is naturally characterized by the policy $\pi$ itself and by its associated value functions $v_\pi$, which capture the consequences of that behavior. 
This motivates the design of a latent space $\mathcal H$ that supports the recovery of both $\pi$ and $v_\pi$ from $ h_\pi$.

We introduce a parametric policy decoder $\pi_\phi(\cdot \mid s, \tilde{h}_\pi)$ that models the action distribution conditioned on state $s$ and policy representation $\tilde{h}_\pi$. 
To handle multiple reward functions, we further map the policy representation space $\mathcal{H}_\theta$ into a collection of latent spaces $\{\mathcal{Z}_{\psi^{(k)}}\}_{k=1}^K$, designed to capture task-specific structures. The mapping is implemented via projection heads $\{u_{\psi^{(k)}}\}_{k=1}^K$ that yield task-specific embeddings $z_\pi^{(k)} := u_{\psi^{(k)}}(\tilde{h}_\pi) \in \mathcal{Z}_{\psi^{(k)}}$, which are subsequently used by value approximators $v_{\xi^{(k)}}(z_\pi^{(k)})$ to predict the corresponding value functions. We denote the projection and value parameters as $\psi=\{\psi^{(k)}\}_{k=1}^K$ and $\xi=\{\xi^{(k)}\}_{k=1}^K$ respectively. 
Together, these components define an encoder--decoder architecture, illustrated in Figure~\ref{fig:architecture}.

\begin{figure}[htbp]
    \centering
    \includegraphics[width=0.9\columnwidth]{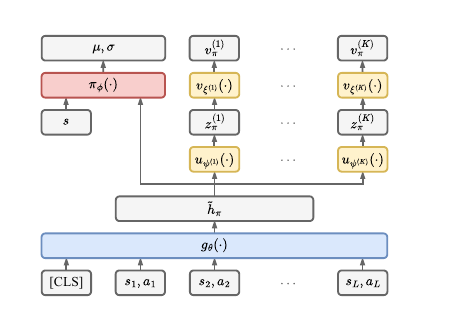}
    \caption{ The model consists of (1) an encoder $g_{\theta}$ that maps sets of state-action pairs to policy representations $\tilde{h}_\pi$; (2) a set of projectors $\{u_{\psi^{(k)}}\}_{k=1}^K$ that map $\tilde{h}_\pi$ to task-specific embeddings $\{z_\pi^{(k)}\}_{k=1}^K$ and a set of value regressors $\{v_{\xi^{(k)}}\}_{k=1}^K$; (3) a policy $\pi_\phi$ that outputs actions given $\tilde{h}_\pi$ and a query state $s$.
    }
    \label{fig:architecture}
\end{figure}

\section{Learning Policy Representations}
\label{sec::problem}

This section describes how to learn policy representations. 
Section~\ref{sec::problem_formulation} introduces the available data and the   learning objective, Sections~\ref{sec::policy_loss} and~\ref{sec::value_loss} detail the policy and value losses, and Section~\ref{sec::practical_considerations} discusses practical considerations.

\subsection{Problem formulation}
\label{sec::problem_formulation}

We assume access to a collection of datasets
\[
    \mathcal{D} = \{\mathcal{D}_{\pi_p}\}_{p=1}^P, 
    \qquad 
    \mathcal{D}_{\pi_p}=\{(s_n,a_n)\}_{n=1}^{N_p},
\]
each obtained by executing one of $P$ policies and collecting $N_p$ state--action samples. 
These samples may be generated through online interaction, expert demonstrations, random policies, or any mixture thereof.
For each policy $\pi_p$, we also observe $M_p$ returns resulting from its execution,
\[
    \mathcal{R} = \{\mathcal{R}_{\pi_p}\}_{p=1}^P, 
    \qquad 
    \mathcal{R}_{\pi_p} = \{R_m\}_{m=1}^{M_p},
\]
where each $R_m$ is a vector of returns across all tasks.
These returns provide supervision signals capturing the consequences of policy behavior under different objectives.
For training, we repeatedly sample context sets from each policy dataset. 
Specifically, for a policy $\pi_p$, we sample a subset of state--action pairs to form a context set $\mathcal{C} \sim \mathcal{D}_{\pi_p}$ and associate it with a single return $R \sim \mathcal{R}_{\pi_p}$. 
Repeating this process yields a collection of training pairs $\{(\mathcal{C}_i, R_i)\}_{i=1}^I$, where each context set provides a finite representation of a policy and the associated return is the supervision signal.

As discussed in the previous section, our objective is to learn the parameters of an encoder--decoder architecture that maps sampled policy executions into a latent space from which both policies and value functions can be reconstructed. Beyond reconstruction, we seek a latent space that supports gradient-based optimization for searching policies that satisfy value-based criteria.
To this end, we endow the latent space $\mathcal{H}_\theta$ with a geometry in which distances reflect differences in value functions, ensuring that local movements in $\mathcal{H}_\theta$ correspond to meaningful changes in policy performance. 
Thus, we consider the following optimization
\begin{equation}
    \label{eq::general_problem}
    \argmin_{\theta, \phi, \psi, \xi} \;\; \ell_\pi(\theta, \phi) + \ell_v(\theta, \psi, \xi),
\end{equation}
where $\ell_\pi(\theta,\phi)$ enforces policy reconstruction, and $\ell_v(\theta,\psi,\xi)$ enforces value reconstruction while inducing a value-aligned geometry on $\mathcal{H}_\theta$.
The following two sections detail the construction of these losses.


\subsection{Policy loss}
\label{sec::policy_loss}

We implement policy encoding and decoding with variational autoencoder (VAE) framework \cite{kingma2013auto}. 
Rather than producing a point estimate, the encoder $g_\theta$ maps a context set $\mathcal{C}_i$ to a variational posterior distribution $q_\theta(\cdot \mid \mathcal{C}_i)$. 
The decoder acts as a conditional stochastic policy: given a query state $s$ and a latent policy representation $\tilde{h}_i \sim q_\theta(\cdot \mid \mathcal{C}_i)$, it outputs the mean  and standard deviation of a diagonal Gaussian distribution over actions.
We minimize a $\beta$-VAE objective \cite{higgins2017beta},
\begin{align}
    \ell_{\pi}(\theta, \phi) = &- \sum_{i=1}^I  \mathbb{E}_{ q_\theta(\cdot \mid \mathcal{C}_i)} \left[ \log \pi_\phi(a \mid s, \tilde{h}_i) \right] \\
    \notag
    & + \beta \sum_{i=1}^I  D_{\mathrm{KL}} \big( q_\theta(\cdot \mid \mathcal{C}_i) \;\|\; \mathcal{N}(0, I) \big).
\end{align}
The reconstruction term ensures that distinct policies admit distinct representations, while the $\beta$-weighted KL regularization promotes a smooth and continuous latent space, preventing fragmentation of the representation manifold.


\subsection{Value function loss}
\label{sec::value_loss}

VAEs have the capacity to imitate diverse policy behaviors, and its variational objective promotes smoothness in the latent space $\mathcal{H}_\theta$ \cite{wang2017robust,duan2017one}.
However, generative modeling alone does not organize the latent space according to the long-term consequences of policy behavior, which are encoded in the value functions. 
Since a single policy may be evaluated under multiple rewards, we introduce task-specific latent spaces $\{\mathcal{Z}_{\psi^{(k)}}\}_{k=1}^K$ via projections from the shared space $\mathcal{H}_\theta$.

Beyond value reconstruction, we seek to endow the shared latent space $\mathcal{H}_\theta$ with a value-function–aligned geometry. 
To this end, we apply contrastive learning in each task-specific latent space $\mathcal{Z}_{\psi^{(k)}}$ and design the projections $u_{\psi^{(k)}}$ to transfer the induced structure back to $\mathcal{H}_\theta$, yielding
\begin{equation}
    \label{eq::value-loss}
    \ell_v(\theta, \psi, \xi)
    =
    \sum_{k=1}^K
    \alpha \ell_v^0(\theta, \psi^{(k)})
    +
    \zeta \ell_v^1(\psi^{(k)})
    +
    \ell_v^2(\xi^{(k)}),
\end{equation}
where $\alpha$ and $\zeta$ are weighting hyper-parameters, with selections detailed in Appendix~\ref{appendix:hyperparameters}. The term $\ell_v^0$ is a contrastive loss defined on $\mathcal{Z}_{\psi^{(k)}}$, and $\ell_v^1$ regularizes the projections to preserve this structure in the shared space $\mathcal{H}_\theta$. 
Additionally, $\ell_v^2$ corresponds to a value-function reconstruction loss, which is particularly relevant for the behavior synthesis problem considered in this work. 
In practice, we optimize $\ell_v^2$ only with respect to the value regressors parameters $\xi$, which can be learned as a downstream task.

\noindent \textbf{Contrastive loss.} We adopt the Rank--N--Contrast (RNC) loss \cite{zha2023rank}, which enforces distance relationships between representations based on task-relevant scalar quantities. 
In our setting, these quantities are  returns, which are stochastic realizations of the value function.
Given a context set $\mathcal{C}_i$, we obtain a policy representation $\tilde{h}_i \sim q_\theta(\cdot \mid \mathcal{C}_i) \in \mathcal{H}_\theta$, which is then projected into a task-specific latent space, yielding $z_i\smash{^{(k)}} \in \mathcal{Z}_{\psi^{(k)}}$. 
Each embedding $z_i\smash{^{(k)}}$ is associated with a scalar return $R_i\smash{^{(k)}}$ corresponding to task $k$. 
Given another embedding $z_j\smash{^{(k)}}$ with associated return $R_j\smash{^{(k)}}$, we define: 
i) the similarity between embeddings as {$s_{i,j}\smash{^{(k)}} := -\| z_i\smash{^{(k)}} -  z_j\smash{^{(k)}} \|_2$
; and 
ii) the return-based distance $ d_{i,j}\smash{^{(k)}} := |R_i\smash{^{(k)}} - R_j\smash{^{(k)}}|.$

For an anchor index $i$ and a target index $j$, we define
\begin{equation}
    \notag
    \mathcal{F}_{i,j}^{(k)}
    :=
    \{\, l \in [I] \setminus \{i\} \mid d_{i,l}^{(k)} \ge d_{i,j}^{(k)} \,\},
\end{equation}
the set of indices whose return-based distance to the anchor is no smaller than that of the target.
Then, the RNC objective minimizes the expected negative log-likelihood: 
\begin{align}
    \notag
    &\ell_v^0(\theta, \psi^{(k)}) = \\
    \notag
    &-
    \sum_{i=1}^I
    \sum_{\substack{j=1 \\ j\neq i}}^I
    \mathbb{E}_{q_\theta(\cdot | \mathcal{C}_i), \; q_\theta(\cdot | \mathcal{C}_j)}
    \left[
        \log
        \frac{\exp(s^{(k)}_{i,j})}
        {\sum_{l \in \mathcal{F}_{i,j}^{(k)}} \exp(s^{(k)}_{i,l})}
    \right].
\end{align}
Intuitively, the RNC loss encourages policies with similar returns to be embedded closer in latent space than those with disparate returns. As shown in \citet[Section~4]{zha2023rank}, optimizing the RNC objective induces $\delta$-ordered representations in the task-specific latent space $\mathcal{Z}_{\psi^{(k)}}$ with respect to returns (see Appendix \ref{app:delta_ordering}), which in turn is key to enable gradient-based optimization in the latent space.

\noindent \textbf{Projector loss.} The role of the projector $u_{\psi^{(k)}}$ is to transfer the $\delta$-ordering induced by the RNC loss from the task-specific latent space $\mathcal{Z}_{\psi^{(k)}}$ back to the shared policy representation space $\mathcal{H}_\theta$. 
While $u_{\psi^{(k)}}$ could in principle be implemented as a non-linear map, preserving this ordering requires a structure that does not distort relative distances.
For this reason, we parameterize $u_{\psi^{(k)}}$ as a linear map $z_i = u_{\psi^{(k)}}(h_i) = U^{(k)} h_i + b^{(k)}$, where $U^{(k)} \in \mathbb{R}^{d_{z_k} \times d_h}$ is constrained to be semi-orthonormal, i.e., $U^{(k)} (U^{(k)})^\top = I_{d_{z_k}}$ with $I_{d_{z_k}}$ being the identity of dimensionality $d_{z_k}$. 
Under this constraint, $u_{\psi^{(k)}}$ defines a partial isometry that preserves Euclidean distances between representations projected onto the row space of $U^{(k)}$. 
As a result, the $\delta$-ordering of per-task representations $\{z_i\}_{i=1}^I$ is preserved in the base policy representations $\{h_i\}_{i=1}^I$.
We formalize this in the following Proposition, with proof given in Appendix~\ref{app::proof_partial_isometry}.

\begin{proposition}
    \label{prop::partial_isometry}
    For semi-orthonormal $U^{(k)}$, the  map $u_{\psi^{(k)}}$ defines a partial isometry. As a result, for any collection of $\delta$-ordered representations $\{z_i\}_{i=1}^I$, the corresponding generating representations $\{h_i\}_{i=1}^I$, when projected onto the row space of $U^{(k)}$, preserve the $\delta$-ordering for all $k$.
\end{proposition}

Therefore, we regularize to encourage this orthonormality property in the linear operator by using the loss
\begin{equation}
    \ell_v^1(\psi^{(k)}) = \| U^{(k)} (U^{(k)})^\top - I_{d_{z_k}} \|_F^2.
\end{equation}

\noindent \textbf{Value reconstruction loss.} To reconstruct the value functions we fit a collection of regressors $\{v_{\xi^{(k)}}\}_{k=1}^K$, where each $v_{\xi^{(k)}}$ is trained to approximate $v^{(k)}_\pi $ by minimizing
\begin{equation}
    \notag
    \ell_v^{2}(\xi^{(k)}) =  \sum_{i=1}^I \mathbb{E}_{q_\theta(\cdot \mid \mathcal{C}_i)} \left[ \left\| v_{\xi^{(k)}}\big( u_{\psi^{(k)}}(\tilde{h}_i) \big) - R_i^{(k)} \right\|_2^2 \right].
\end{equation}

Proposition~\ref{prop::partial_isometry} guarantees that representations in the shared latent space $\mathcal{H}_\theta$ are $\delta$-ordered up to a projection. 
This implies that the latent space is arranged in such a way that neighboring policies in the  representation spaces yield similar returns, that are a proxy to value functions.
As a consequence, local gradients in the shared latent space $\mathcal{H}_\theta$ correspond to monotone changes in value functions, enabling reliable gradient-based optimization directly in representation space to search for policies with prescribed value-function properties. 
This observation is central to the behavior synthesis task introduced next.



\subsection{Practical considerations}
\label{sec::practical_considerations}

\noindent \textbf{Sampling.} In practice, we run trajectories under different policies and we treat every observed trajectory as a distinct policy realization, hence $M_p = 1$ and $N_p$ is the number of transitions within the trajectory. A context set is constructed by sampling unordered subsets of state--action pairs from a $D_{\pi_p}$. This allows us to construct positive pairs for contrastive learning by drawing two independent context sets from the same $D_{\pi_p}$. See Algorithm~\ref{appendix:algorithm} for details.

\noindent \textbf{Parametrization.} Motivated by Theorem~\ref{thm:existence_vs_tilde}, we parameterize the encoder $g_\theta$ using an attention-based architecture, specifically a BERT-style Transformer \cite{devlin2019bert}. 
Each context set $\mathcal{C}$ is treated as a sequence of $(s,a)$ tokens augmented with a learnable \texttt{[CLS]} token, whose final embedding defines the policy representation posterior $p_\theta(\cdot \mid \mathcal{C})$. 
To ensure permutation invariance, we omit positional embeddings and treat each context set as an unordered collection. The decoder $\pi_\phi$ defines a stochastic policy with a diagonal Gaussian action distribution, whose mean and standard deviation are given by MLP-based networks. See Appendix~\ref{appendix:trainingdetails} for details.

\section{Steerable Behavior Synthesis}
\label{sec::inference}

Our framework learns a structured policy latent space $\mathcal{H}_\theta$ where the $\delta$-ordering property of the latent space hints that we can search in the latent space with gradient-based optimization.
This enables zero-shot behavior synthesis by directly optimizing over latent representations at test time. 
Unlike conditional generative models that can only handle conditioning variables seen at training time~\cite{chen2021decision,ajay2022conditional}, we treat behavior synthesis as a constrained optimization problem over $h\in\mathcal{H}_\theta$. This formulation allows searching the policy space to reach performance targets subject to previously unseen constraints without additional training or environment interaction.

Given a target value $v_g\smash{^{(1)}}$ for objective $k=1$ and additional inequality constraints on other objectives, we solve
\begin{equation}
\label{eq:latent_constrained}
\begin{aligned}
\min_{h\in\mathcal{H}_\theta}\;\;&
\ell_{\mathrm{g}}(h)
:=
\bigl\| v_{\xi^{(1)}}(u_{\psi^{(1)}}(h)) - v_g^{(1)} \bigr\|_2^2 \\
\text{s.t.}\;\;&
c^{(k)}(h) \le 0,\qquad k=2,\dots,K,
\end{aligned}
\end{equation}
where each constraint $c^{(k)}(h)$ can encode requirements such as safety bounds expressed through learned value predictors, e.g., 
$c^{(k)}(h)= v^{(k)}_{\mathrm{c}} - v_{\xi^{(k)}}(u_{\psi^{(k)}}(h))$.

While similar latent optimization strategies have been explored for the unconstrained case \cite{gomez2018automatic}, naive gradient descent in high-dimensional latent spaces risks drifting into low-density regions, where the decoder and value predictors are not well-supported. 
To keep iterates on the support of the learned representation manifold, we project each update onto a local tangent approximation.
Specifically, at each iterate $t$, we construct a matrix $E(h_t)$ with the nearest neighbors of $h_t$ among training policy embeddings; compute $V_t\in\mathbb{R}^{d_h\times p}$, which contains the top-$p$ principal components of $E(h_t)$; and define the projector $P_t := V_t V_t^\top$.

We solve \eqref{eq:latent_constrained} using a projected primal--dual method. 
With Lagrange multipliers $\lambda\in\mathbb{R}_+^{K-1}$, define the Lagrangian $\mathcal{L}(h,\lambda)=\ell_{\mathrm{steer}}(h)- \lambda^\top c(h)$, where $c(h)=[c_1(h),\dots,c_L(h)]^\top$. Then, we run $T$ iterations as
\begin{align}
h_{t+1}
&\leftarrow
h_t
-
\eta_h\, P_t \nabla_h \mathcal{L}(h_t,\lambda_t),
\label{eq:latent_primal}\\
\lambda_{t+1}
&\leftarrow
\Bigl[\lambda_t + \eta_\lambda\, c(h_{t+1})\Bigr]_+,
\label{eq:latent_dual}
\end{align}
where $[\cdot]_+$ denotes projection onto the non-negative orthant.
After optimization, the synthesized policy is obtained by decoding $\pi_\phi(\cdot\mid s,h_T)$.

\begin{figure*}[h!]
    \centering
    \begin{subfigure}[b]{0.44\textwidth}
        \centering
        \includegraphics[width=\textwidth]{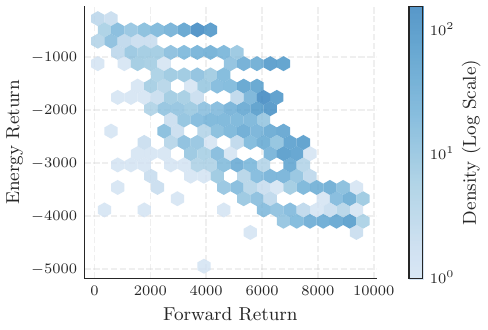}
    \end{subfigure}
    \hspace{5mm}
    \begin{subfigure}[b]{0.44\textwidth}
        \centering
        \includegraphics[width=\textwidth]{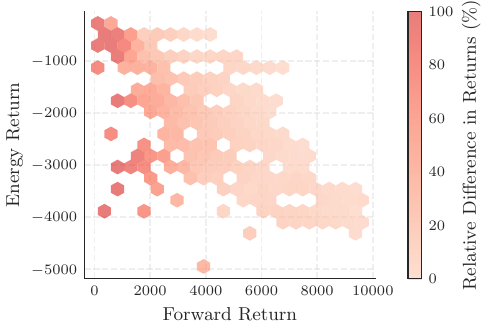}
    \end{subfigure}
    
    \caption{(a) Distribution of policies in the training dataset, where color intensity represents the number of trajectories falling within each return bin. (b) Relative difference in returns between ground-truth  and returns obtained by the latent-conditioned decoder. 
    }
    \label{fig:side_by_side_comparison}
\end{figure*}

\section{Experiments}
\label{sec::results}
Our experiments evaluate (i) the fidelity of policy reconstruction from latent representations; (ii) the emergence of ordered latent structures induced by RNC loss; and (iii) the efficacy of zero-shot behavior synthesis under constraints.
To such end, we use the Multi-Objective MuJoCo environments \cite{momujoco}, which extend standard continuous control tasks \cite{mujoco} with vector-valued rewards. 
Learning a policy representation space requires training data that spans diverse behaviors. 
Therefore, we construct our datasets by first training different PPO agents \cite{schulman2017proximal}, each optimized under a distinct scalarization of the reward vector \cite{roijers2013survey}, and subsequently collecting trajectories from both intermediate and converged policies. Figure~\ref{fig:side_by_side_comparison} (left) visualizes the resulting return landscape for the HalfCheetah environment. See Appendix~\ref{appendix:datasets} for more details on dataset generation.

\subsection{Policy reconstruction from representations}
Successful behavior steering via latent space optimization requires the encoder to map diverse policies to distinct representations and the decoder to faithfully reconstruct. 
We validate these capabilities by deploying trained decoder conditioned on representations $\tilde{h}_i \sim q_\theta(\cdot \mid \mathcal{C}_i)$ and comparing rollout returns $\tilde{R}_i$ against target returns $R_i$. 
Figure~\ref{fig:side_by_side_comparison} (right) illustrates that for HalfCheetah, relative differences remain largely below $10\%$. These results confirm our framework's capacity to distinguish and reproduce diverse behaviors. Results for other environments are provided in Appendix~\ref{appendix:imitation}.

\subsection{Structure of latent space}

We first qualitatively evaluate the alignment between latent geometry and return landscapes via dimensionality reduction.
Projected embeddings $\{z_i\smash{^{(k)
}}\}_{i=1}^I$ are driven to be $\delta$-ordered with respect to $R^{(k)}$ by RNC loss. UMAP \cite{mcinnes2018umap} visualizations in Figure~\ref{fig:zspace} reveal emergence of smooth and monotonic representations.
Crucially, this geometry transfers to the shared representation $\mathcal{H}_\theta$, as it can be seen in Figure~\ref{fig:repr_comparison} (left), yielding an organized space amenable for gradient-based search.
In contrast, representations learned with a standard VAE, presented in Figure~\ref{fig:repr_comparison} (right) show significant disorder with respect to returns, indicating that reconstruction objectives alone are insufficient to align the latent space with the return landscape.

\begin{figure}[htbp]
    \centering
    \includegraphics[width=1.0\columnwidth]{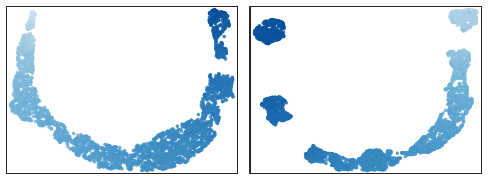}
    \caption{UMAP visualizations of the projected embeddings $\{z_i\smash{^{(k)}}\}_{i=1}^I$ colored with corresponding velocity returns (Left) and energy penalty returns (Right). The learned embeddings exhibit a consistent ordering.}
    \label{fig:zspace}
\end{figure}

\begin{figure*}[t!]
  \centering
  \begin{subfigure}[b]{0.45\textwidth}
    \includegraphics[width=\textwidth]{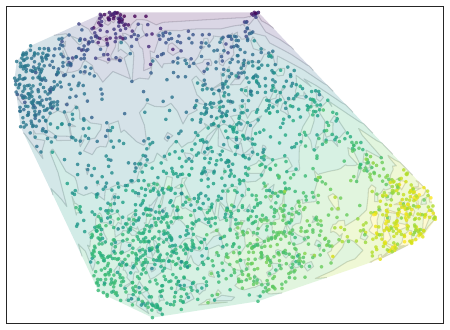}
  \end{subfigure}
  \hfill
  \begin{subfigure}[b]{0.45\textwidth}
    \includegraphics[width=\textwidth]{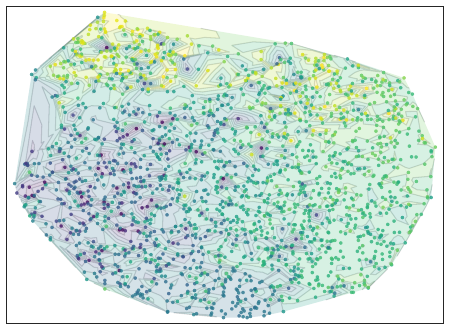}
  \end{subfigure}
  \hfill
  \begin{subfigure}[b]{0.09\textwidth}
    \includegraphics[height=5.65cm]{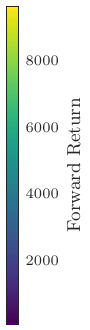}
  \end{subfigure}
  \caption{UMAP visualizations of (a) policy representations trained with our full method, underlayed with contours of return values; and (b) policy representations trained with VAE baseline ($\alpha = 0$), showing a disordered landscape with no clear value-based geometry.}
  \label{fig:repr_comparison}
\end{figure*}

We quantitatively evaluated this alignment by training linear regressors to predict normalized returns with the encoder frozen (Table~\ref{tab:linear_probing}). 
We compare against (1) \textit{VAE ($\alpha=0$)}: a standard VAE trained without contrastive loss, similar to \citet{wang2017robust}; (2) \textit{MLP Mean-Pool}: state-action pairs are processed independently with an MLP and aggregated via average pooling, effectively a Monte Carlo estimator; and (3) \textit{Unconstrained Projector ($\gamma=0$)}: our architecture trained without orthonormal regularization on the projection heads. Our superiority over \textit{MLP Mean-Pool} validates the use of attention-based encoder and supports Theorem~\ref{thm:existence_vs_tilde}, while improvement over \textit{Unconstrained Projector} suggests that orthonormal constraint enforces structure in $\mathcal{H}_\theta$ rather than allowing the projection head to absorb the complexity.

\begin{table}[htbp]
\centering
\caption{MSE of regressing policy representations $\tilde{h}$ into normalized trajectory returns using trained embeddings. Test set consists of trajectories collected by policies not included in the training set.}
\label{tab:linear_probing}
\resizebox{\columnwidth}{!}{%
\begin{tabular}{l c c}
\toprule
\textbf{Method} & \textbf{Train MSE} $\downarrow$ & \textbf{Test MSE} $\downarrow$ \\
\midrule
\textbf{Ours (Full Method)} & $\mathbf{0.029}$ & $\mathbf{0.201}$ \\
\midrule
VAE ($\alpha=0$) & $0.233$ & $0.850$ \\
MLP Mean-Pool Encoder & $0.081$ & $0.289$ \\
Unconstrained Projector ($\gamma=0$) & $0.053$ & $0.232$ \\
\bottomrule
\end{tabular}
}
\end{table}


\subsection{Zero-Shot Constrained Behavior Synthesis}

To demonstrate the merit of structured representations, we solve a novel constraint-aware behavior synthesis task via latent-space constrained optimization. 
We consider a HalfCheetah agent, targeting a specific forward return $v_g\smash{^{(1)}}$ subject to an energy constraint $v_c\smash{^{(2)}}$. The goal is to find a latent code $h^*$ that minimizes the steering objective $\ell_{\text{steer}}$ for target $v_g^{(1)}$ while strictly satisfying $v^{(2)} \ge v_{c}\smash{^{(2)}}$. Optimization runs entirely in the latent space via primal-dual updates (Eqs.~\ref{eq:latent_primal}--\ref{eq:latent_dual}) using the trained value predictors as differentiable surrogates for the true value functions.

Figure~\ref{fig:constrained_steering} visualizes a representative task initialized at a high-velocity but energy-inefficient policy, with $v^{(1)}_{\text{init}} \approx 8000$ and $v^{(2)}_{\text{init}} \approx -4000$. Querying for $v_{\text{g}}^{(1)} = 6000$ subject to $v^{(2)} \ge v_c^{(2)} = -3000$, the trajectory rapidly exits the infeasible region and crosses the constraint boundary as shown in Figure~\ref{fig:constrained_steering} (right). Subsequently, the primal-dual mechanism ensures the trajectory evolves along the boundary edge without re-entering the infeasible region. Concurrently, Figure~\ref{fig:constrained_steering} (left) shows the trajectory converging to the target iso-contour. The optimization terminates once a latent vector is identified whose forward return is sufficiently close to the prescribed target while satisfying the energy constraint.

\begin{table}[b]
\centering
\caption{Optimization success rates, return discrepancy (target vs. realized) and constraint violation across 100 evaluation tasks.}
\label{tab:quantitative_bench}
\fontsize{9}{9}\selectfont
\begin{tabular}{lccc}
\toprule
\textbf{Method} & $\mathcal{S}(\%)\uparrow$& $\mathcal{E}_{\text{targ}}$ ($\%$) $\downarrow$ & $\mathcal{E}_{\text{cons}}$ ($\%$) $\downarrow$\\
\midrule
\textbf{Ours (Full Method)} & $\mathbf{98.0}$ & $\mathbf{11.67}$ & $\mathbf{0.13}$ \\
\midrule
VAE ($\alpha=0$) & $11.0$ & $40.8$ & $44.13$\\
VAE + InfoNCE & $55.0$ & $25.47$ & $8.93$ \\
AE + RNC & $62.0$ & $12.44$ & $3.56$ \\
\bottomrule
\end{tabular}
\end{table}

Quantitatively, we evaluate $100$ synthesis tasks, each with a randomly sampled $\tilde{h}_{\text{init}}$, forward return target $v_g\smash{^{(1)}}$ and energy constraint $v_c\smash{^{(2)}}$. We report \textit{optimization success} ($\mathcal{S}$) as the percentage of trials where the optimizer converges to a feasible solution according to value predictors. We further collect rollouts with decoder conditioned on final latent codes, reporting \textit{target error} ($\mathcal{E}_{\text{targ}}$; deviation from $v_g\smash{^{(1)}}$) and \textit{Constraint Violation} ($\mathcal{E}_{\text{cons}}$; magnitude of $v^{(2)}$ breach).
We attempt to solve the same optimization problem within the latent spaces learned by (1) \textit{VAE ($\alpha=0$)}: a standard VAE trained without contrastive loss, similar to \citet{wang2017robust}; (2) \textit{VAE + InfoNCE}: replacing RNC loss with InfoNCE \cite{oord2018representation}; and (3) \textit{AE + RNC}: a deterministic autoencoder trained with RNC loss.

Superiority over \textit{VAE + InfoNCE} suggests that clustering policies without explicit ordering is insufficient for reliable latent-space steering. Meanwhile, the lower success rate of \textit{AE + RNC} highlights that the variational prior is crucial for inducing a smooth and connected manifold suitable for gradient-based optimization.

\begin{figure}[htbp]
    \centering
    \includegraphics[width=1.0\columnwidth]{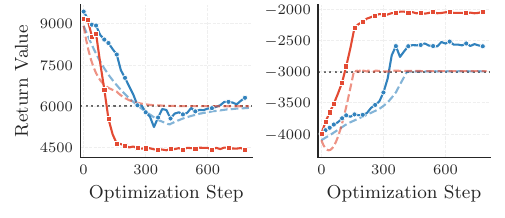}
    \caption{We compare latent optimization trajectories for a constrained synthesis task using naive gradient descent (Red) versus projected gradient descent (Blue). The plots show the forward return (Left) and energy return (Right) of trajectories collected by the decoder (solid lines) and the corresponding value predictions (dashed lines) over optimization steps.}
    \label{fig:projected_gd}
\end{figure}

\begin{figure*}[t]    
    \centering
    \begin{subfigure}[b]{0.49\textwidth}
        \includegraphics[width=\textwidth]{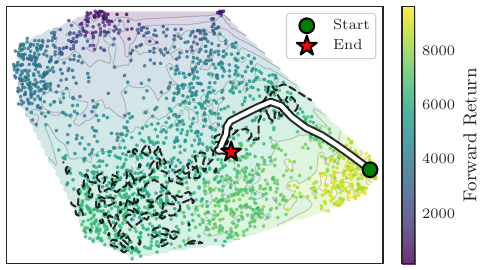}
    \end{subfigure}
    \hfill
    \begin{subfigure}[b]{0.49\textwidth}
        \includegraphics[width=\textwidth]{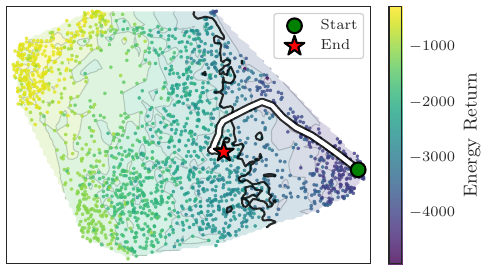}
    \end{subfigure}

    \caption{Optimization trajectory (green dot to red star) overlaid on return landscapes. (a) Trajectory progresses toward the iso-contour (dashed line) corresponding to the target value. (b) Trajectory traverses across the feasible boundary (solid line) via primal-dual updates to satisfy the energy constraint.
    }
    \label{fig:constrained_steering}
\end{figure*}

Finally, restricting gradient updates to the local tangent space of the training distribution is critical for robust synthesis. To demonstrate this, we solve a behavior synthesis task using two distinct optimization strategies: naive gradient descent and our proposed projected gradient descent. We then collect trajectories using the decoder conditioned on the intermediate latent codes along each optimization path, recording both the ground-truth returns and the estimates from the value regressor. The discrepancy between predicted value and actual returns in Figure~\ref{fig:projected_gd} indicates that naive optimization traverses regions where return regressors become unreliable and decoder $\pi_\phi$ fails to reconstruct coherent behaviors, demonstrating the importance of constraining updates to the data manifold.

\section{Discussion \& Related Work}
\label{sec::discussion}

Latent policy representations have been used across diverse domains, including skill discovery \cite{eysenbach2018diversity}, policy evaluation \cite{harb2020policy, tang2022inputting, scarpellini2023pi2}, multi-agent games \cite{grover2018learning, albrecht2018autonomous}, and multi-task learning \cite{wang2017robust, rakelly2019efficient, zintgraf2019varibad, raileanu2020fast, sodhani2021multi}. 
The interpretation of these latent representations varies significantly. In unsupervised skill discovery, representations often encode state reachability \cite{eysenbach2018diversity}, whereas in meta-RL, they encode reward functions for generalization among tasks~\cite{rakelly2019efficient}. In contrast, our representations focuses explicitly on summarizing the policy itself as well as the associated value functions.

A critical design choice in representation learning is the input modality. 
While direct encoding of raw policy parameters has been explored \cite{harb2020policy, tang2022inputting}, such methods are inherently restricted to parametric policies and require identical network architectures.
To achieve broader applicability, a more general paradigm involves encoding environmental interactions, specifically sets of state-action pairs \cite{wang2017robust, grover2018learning, scarpellini2023pi2}. We adopt the interaction-based paradigm to ensure our method remains architecture-agnostic and applicable to any setting where behavior is observable, ranging from active online agents to static offline datasets.


To learn the representation encoder, prior works typically employ either generative or contrastive objectives. Contrastive methods \cite{grover2018learning, tang2022inputting} maximize discriminability but often yield fragmented latent spaces, effectively reducing the representation problem to a classification task. 
Conversely, variational approaches model policies as distributions, promoting latent smoothness but remaining agnostic to the underlying value function profiles \cite{wang2017robust}. Our framework integrates the generative backbone with a specialized contrastive objective \cite{zha2023rank}, inducing a latent geometry explicitly aligned with value functions.

This structured geometry enables a novel behavior synthesis paradigm. Recent conditional generative models \cite{chen2021decision, ajay2022conditional} synthesize behaviors by conditioning the generation process on attribute labels observed during training. Consequently, their test-time capabilities are restricted to condition variables and constraint labels available in the training set. In contrast, we treat behavior synthesis as a test-time optimization problem. This allows us to satisfy novel composite queries without requiring constraints to be observed or explicitly labeled during training.


\section{Conclusion \& Future Work}
\label{sec::conclusion}

We presented a principled framework for learning general-purpose policy representations. By integrating variational generative modeling with contrastive  learning, we proposed a method to design a latent space that is both smooth and explicitly aligned with value functions. This structure enables test-time behavior synthesis under unseen constraints through latent space optimization. Our findings suggest that such value-aware representations are a critical enabler for controllable agent behavior. Interesting future directions include utilizing the latent space to facilitate online RL exploration and exploitation, and developing dynamics-augmented latent-conditioned decoders for meta-learning across varying environments.

\bibliography{reference}
\bibliographystyle{icml2026}

\newpage
\appendix
\onecolumn

\section{Theoretical Results}
\subsection{Proof of Theorem~\ref{thm:existence_vs_tilde}}
\label{app::thm_existence}

\begin{proof}
Recall that the dataset $\mathcal D_\pi=\{(s_n,a_n)\}_{n=1}^N$ sampled i.i.d. from $d_\pi$, and that the policy representation $h_\pi=\mathbb E_{d_\pi}[f(s,a)]$ is empirically and parametrically approximated as
\begin{align}
    \notag
    g(\mathcal D_\pi)&=\sum_{n=1}^N w(s_n,a_n\mid\mathcal D_\pi)\,f(s_n,a_n),\\
    \notag
    g_\theta(\mathcal D_\pi)
    &=
    \sum_{n=1}^N
    w_{\theta_0}(s_n,a_n\mid\mathcal D_\pi)\,f_{\theta_1}(s_n,a_n) = \tilde h_\pi
\end{align}
and we let the estimator $g(\mathcal{D}_\pi)$ be the one in \citep[Lemma 3]{oates2017control}.
We begin the proof by adding and subtracting $g(\mathcal D_\pi)$ and using the triangle inequality to yield
\begin{equation}
\label{eq:tri_existence}
\|\tilde h_\pi-h_\pi\|
\le
\underbrace{\|g(\mathcal D_\pi)-h_\pi\|}_{(a)}
+
\underbrace{\|\tilde h_\pi-g(\mathcal D_\pi)\|}_{(b)},
\end{equation}
where (a) is the sampling error term while (b) is the approximation error.

\noindent \textbf{(a) Sampling error.} Regarding the sampling error term (a) in \eqref{eq:tri_existence}, under Assumptions~i)--iii) of Theorem~\ref{thm:existence_vs_tilde} the conditions required to invoke Theorem~2 of \cite{oates2017control} are satisfied. In particular, the continuous differentiability  of the density associated with $d_\pi$ on $\Omega$ for all $\pi \in \Pi$ ensures A1; the vanishing of the density associated with $d_\pi$ on the boundary of $\Omega$ ensures A2$'$; and the twice differentiability and boundedness of the kernel $\kappa$, together with the boundedness and compactness of $\Omega$, ensure A3--A4. 
Finally, A5 holds since the feature map $f$ is assumed to lie in the RHKS space associated with $\kappa$.
As a consequence, combining Theorem~2 and Lemma~3 of \cite{oates2017control} guarantees the existence of a convex combination of evaluations of $f$ over $N$ state--action samples from $\mathcal{D}_\pi$ such that
\begin{equation}
\label{eq:cf_rate}
\|g(\mathcal D_\pi)-h_\pi\| \le C N^{-1},
\end{equation}
for some constant $C>0$ independent of $\pi$.

Moreover, as noted in Remark~5 of \cite{oates2017control}, the resulting weights $w$: i) depend only on kernel evaluations over $\mathcal{D}_\pi$ so they are independent of the feature function $f$, allowing them to be transferred across evaluations at different state--action pairs; ii) are jointly continuous in $(s,a)$ and in the elements of the set $\mathcal D$, and are permutation-invariant with respect to the samples in $\mathcal D$.

\paragraph{(b) Approximation error.} We now bound the second term (b) in~\eqref{eq:tri_existence}.
We add and subtract the cross-product of the parametric weights with the true feature map 
$\sum_{n=1}^N w_{\theta_0}(s_n,a_n\mid\mathcal D_\pi)\,f(s_n,a_n)$ and use the triangular inequality to obtain
\begin{align}
    \label{eq:approx_split}
    \|\tilde h_\pi-g(\mathcal D_\pi)\| 
    &\le
    \underbrace{\left\|
    \sum_{n=1}^N
    \bigl(
    w_{\theta_0}(s_n,a_n\mid\mathcal D_\pi)-w(s_n,a_n\mid\mathcal D_\pi)
    \bigr)\,f(s_n,a_n)
    \right\|}_{(b.1)} \nonumber \\
    &+
    \underbrace{\left\|
    \sum_{n=1}^N
    w_{\theta_0}(s_n,a_n\mid\mathcal D_\pi)
    \bigl(
    f_{\theta_1}(s_n,a_n)-f(s_n,a_n)
    \bigr)
    \right\|}_{(b.2)}.
\end{align}

Regarding the approximation error of the point-wise feature map (b.2) in \eqref{eq:approx_split}, we now prove that it can be made arbitrarily small by using a universal approximator.
Assuming there exists $C_{f} < \infty$ such that $\|\sum_{n=1}^N w_{\theta_0}(s, a | \mathcal{D})\| \leq C_f$ for all $(s, a) \in \Omega$ and $\mathcal{D} \subset \Omega$ we have that
\begin{align}
    \left\|
    \sum_{n=1}^N
    w_{\theta_0}(s_n, a_n | \mathcal{D}_\pi)
    \bigl(
    f_{\theta_1}(s_n, a_n)-f(s_n, a_n)
    \bigr)
    \right\|
    &\le
    \sum_{n=1}^N
    |w_{\theta_0}(s_n,a_n\mid\mathcal D_\pi)|
    \|f_{\theta_1}(s_n,a_n)-f(s_n,a_n)\|\nonumber\\
    &\le
    C_f
    \sup_{(s,a) \in \Omega}
    \|f_{\theta_1}(s,a)-f(s,a)\| .
    \label{eq:feature_bound_ex}
\end{align}
Fix $\varepsilon>0$. Since $\Omega$ is compact and $f$ is continuous, by standard universal
approximation results for MLPs \citep[Theorem 1]{cybenko1989approximation} there exists $\theta_1$ such that
\begin{equation}
\label{eq:feature_approx_ex}
    \sup_{(s,a) \in \Omega}
    \|f_{\theta_1}(s,a)-f(s,a)\|
    \le
    \frac{\varepsilon}{2C_f}.
\end{equation}

Now we do the same for the approximation of the weights term (b.1) in \eqref{eq:approx_split}.
Since $\Omega$ is compact and $\kappa$ is bounded, $f$ is bounded on $\Omega$.
Thus there exists $C_w<\infty$ such that $\sum_{n=1}^N\|f(s_n,a_n)\|\le C_w$ for all $(s,a)\in\Omega$, yielding
\begin{align}
    \left\|
    \sum_{n=1}^N
    \bigl(
    w_{\theta_0}(s_n,a_n\mid\mathcal D_\pi)-w(s_n,a_n\mid\mathcal D_\pi)
    \bigr) f(s_n, a_n)
    \right\| &\le
    \sum_{n=1}^N
    \bigl|
    w_{\theta_0}(s_n,a_n\mid\mathcal D_\pi)
    -
    w(s_n,a_n\mid\mathcal D_\pi)
    \bigr| \bigl| f(s_n, a_n)\bigr|, \nonumber \\
    &\le
    C_w \sup_{(s,a) \in \Omega, \mathcal{D} \subset \Omega}
    \bigl|
    w_{\theta_0}(s,a\mid\mathcal D)
    -
    w(s,a\mid\mathcal D) \bigr|.
    \label{eq:weight_bound_ex}
\end{align}
For fixed $(s,a)\in\Omega$, the mapping $w$ is a continuous permutation-invariant function of the finite set $\mathcal D \subset \Omega$.
By universality of set-based approximators for continuous invariant set functions \citep[Theorem~2]{zaheer2017deep}, there exist MLPs $\rho_{\theta_0}^0$ and $\rho_{\theta_0}^1$ such that the dependence of $w(s,a\mid\mathcal D)$ on $\mathcal D$ can be uniformly approximated by the invariant summary
\[
S(\mathcal{D}) = \rho_{\theta_0}^1\!\left(\sum_{(s',a')\in\mathcal D}\rho_{\theta_0}^0(s',a')\right),
\]
for all $\mathcal D$ with $|\mathcal D|=N$.
Since $w$ depends jointly and continuously on $(s,a)$ and $S(\mathcal D)$, universality of feedforward networks on compact domains \citep[Theorem~1]{cybenko1989approximation} implies the existence of an MLP $\rho_{\theta_0}^2$ such that
\[
w_{\theta_0}(s,a\mid\mathcal D)
=
\rho_{\theta_0}^2\!\Big((s,a),\,S(\mathcal D)\Big)
\]
uniformly approximates $w(s,a\mid\mathcal D)$ over $(s,a)\in\Omega$ and all $\mathcal D\subset\Omega$ with $|\mathcal D|=N$.
Therefore, there exists $\theta_0$ such that
\begin{equation}
\label{eq:weight_approx_ex}
    \sup_{\substack{(s,a) \in \Omega\\ \mathcal{D} \subset \Omega,\ |\mathcal D|=N}}
    \bigl|
    w_{\theta_0}(s,a\mid\mathcal D)
    -
    w(s,a\mid\mathcal D)
    \bigr|
    \le
    \frac{\varepsilon}{2 C_w}.
\end{equation}

Now, substituting~\eqref{eq:weight_approx_ex} and~\eqref{eq:feature_approx_ex}
into~\eqref{eq:weight_bound_ex}--\eqref{eq:feature_bound_ex}, we obtain
\[
    \|\tilde h_\pi - g(\mathcal D_\pi)\|
    \le
    \varepsilon.
\]
Combining with~\eqref{eq:cf_rate} and~\eqref{eq:tri_existence} yields
\[
    \|\tilde h_\pi-h_\pi\|
    \le
    C N^{-1}+\varepsilon,
\]
which proves~\eqref{eq:existence_vs_tilde}.
\end{proof}

\subsection{Proof of Proposition \ref{prop::partial_isometry}}
\label{app::proof_partial_isometry}

\begin{proof}
Recall that for $h \in \mathcal{H}_\theta$, the projector $u_{\psi^{(k)}}$ maps $\mathcal{H}_\theta$ into $\mathcal{Z}_{\psi^{(k)}}$ as an affine map $z^{(k)} = u_{\psi^{(k)}}(h) = U^{(k)} h + b^{(k)}$, where $U^{(k)} \in \mathbb{R}^{d_{z_k} \times d_h}$ is assumed to be semi-orthonormal, i.e., $U^{(k)} (U^{(k)})^\top = I_{d_{z_k}}$. 
Since $U^{(k)}$ is semi-orthonormal, $W^{(k)} = (U^{(k)})^\top U^{(k)}$ is the orthogonal projector onto the row space of $U^{(k)}$. 
Because Euclidean distance is translation invariant, the bias $b^{(k)}$ cancels out, yielding
\begin{equation}
    \label{eq::partial_isometry_mapping}
    \|z^{(k)}_i - z^{(k)}_j\|_2 = 
    \|U^{(k)}(h_i - h_j)\|_2 = \| W^{(k)}(h_i - h_j)\|_2
\end{equation}
for all $k$.
Therefore, $u_{\psi^{(k)}}$ defines a partial isometry, in the sense that it preserves Euclidean distance between representations projected onto the row space of $U^{(k)}$.

Since we have $\{z^{(k)}_i\}_{i=1}^I$ that are $\delta$-ordered for all $k$ (i.e., with score $s^{(k)}_{i,j}:=-\|z_i^{(k)}-z_j^{(k)}\|_2$),
it follows that for any $0<\delta<1$ and any $i\in[I]$, $j,l\in[I]\setminus\{i\}$,
\[
\begin{cases}
    \|z_i^{(k)} - z_j^{(k)}\|_2 
    < 
    \|z_i^{(k)} - z_l^{(k)}\|_2 - \frac{1}{\delta}
    & \text{if } d_{i,j} < d_{i,l}, \\[6pt]
    \bigl|
    \|z_i^{(k)} - z_j^{(k)}\|_2 
    -
    \|z_i^{(k)} - z_l^{(k)}\|_2
    \bigr|
    < \delta
    & \text{if } d_{i,j} = d_{i,l}, \\[6pt]
    \|z_i^{(k)} - z_j^{(k)}\|_2 
    >
    \|z_i^{(k)} - z_l^{(k)}\|_2 + \frac{1}{\delta}
    & \text{if } d_{i,j} > d_{i,l}.
\end{cases}
\]
Therefore, leveraging \eqref{eq::partial_isometry_mapping}, the same inequalities hold in
$\mathcal{H}_\theta$ under the projected geometry induced by $W^{(k)}$
\[
\begin{cases}
    \|W^{(k)}(h_i - h_j)\|_2 
    <
    \|W^{(k)}(h_i - h_l)\|_2 - \frac{1}{\delta}
    & \text{if } d_{i,j} < d_{i,l}, \\[6pt]
    \bigl|
    \|W^{(k)}(h_i - h_j)\|_2
    -
    \|W^{(k)}(h_i - h_l)\|_2
    \bigr|
    < \delta
    & \text{if } d_{i,j} = d_{i,l}, \\[6pt]
    \|W^{(k)}(h_i - h_j)\|_2
    >
    \|W^{(k)}(h_i - h_l)\|_2 + \frac{1}{\delta}
    & \text{if } d_{i,j} > d_{i,l}.
\end{cases}
\]
Consequently, representations $\{h_i\}_{i=1}^I$ inherits the $\delta$-ordering property from
$\{z_i^{(k)}\}_{i=1}^I$ along the row-space of $U^{(k)}$, i.e., up to orthogonal projection by
$W^{(k)}$. 
\end{proof}

\section{Datasets}\label{appendix:datasets}

Training the proposed architecture requires datasets that consist of trajectories with diverse behaviors. Standard offline RL benchmarks such as D4RL \cite{fu2020d4rl} are not suitable for our use case as they typically comprise trajectories collected by a limited number of policies, lacking the necessary diversity across the objective space.

\paragraph{Multi-Objective MuJoCo Environments.} 
We collect a dataset of policies with diverse behaviors using the Multi-Objective MuJoCo variants \cite{momujoco}. In these environments, the agent observes a reward vector at each transition rather than a scalar value. The specific environments and their definitions are detailed below. Note that alive bonuses $C$ and contact force penalties $F_{\text{contact}}$ are incorporated into the objectives to facilitate policy learning:

\begin{itemize}
    \item \textbf{MO-HalfCheetah-v5:} With state space $\mathcal{S} \in \mathbb{R}^{17}$ and action space $\mathcal{A} \in \mathbb{R}^{6}$, the agent considers two objectives: forward speed $r^{(1)} = v_x$ and energy consumption $r^{(2)} = -\sum_{i}a_i^2$, where $v_x$ is the velocity in the $x$-direction and $a_i$ is the torque applied to the $i$-th joint.
    
    \item \textbf{MO-Walker2d-v5:} With state space $\mathcal{S} \in \mathbb{R}^{17}$ and action space $\mathcal{A} \in \mathbb{R}^{6}$, the agent considers two objectives: forward speed $r^{(1)} = v_x + C$ and energy consumption $r^{(2)} = -\sum_{i}a_i^2 + C$. Here, $C = 1$ is the alive bonus, $v_x$ is the $x$-direction velocity, and $a_i$ is the applied joint torque.

    \item \textbf{MO-Humanoid-v5:} With state space $\mathcal{S} \in \mathbb{R}^{348}$ and action space $\mathcal{A} \in \mathbb{R}^{17}$, the agent considers two objectives: forward speed $r^{(1)} = v_x - F_{\text{contact}} + C$ and energy consumption $r^{(2)} = -\sum_{i}a_i^2 - F_{\text{contact}} + C$. Here, $C = 5$ is the alive bonus, $F_{\text{contact}}$ is the cost of external contact forces, $v_x$ is the $x$-direction velocity, and $a_i$ is the applied joint torque.
    
    \item \textbf{MO-Hopper-v5:} With state space $\mathcal{S} \in \mathbb{R}^{11}$ and action space $\mathcal{A} \in \mathbb{R}^{3}$, the agent considers three objectives: forward speed $r^{(1)} = v_x + C$, jumping height $r^{(2)} = 10 \cdot z + C$, and energy consumption $r^{(3)} = -\sum_{i}a_i^2 + C$. Here, $C = 1$ is the alive bonus, $z$ denotes the vertical height of the agent, $v_x$ is the $x$-direction velocity, and $a_i$ is the applied joint torque.

    \item \textbf{MO-Hopper-2obj-v5:} A modified version of MO-Hopper-v5 with identical state-action spaces but only two objectives. The energy penalty is integrated into the motion objectives: forward speed $R^{(1)} = v_x -\sum_{i}a_i^2 + C $ and jumping height $r^{(2)} = z -\sum_{i}a_i^2 + C$. 

    \item \textbf{MO-Ant-v5:} With state space $\mathcal{S} \in \mathbb{R}^{105}$ and action space $\mathcal{A} \in \mathbb{R}^{8}$, the agent considers three objectives: $x$-axis velocity $r^{(1)} = v_x - F_{\text{contact}} + C$, $y$-axis velocity $r^{(2)} = v_y - F_{\text{contact}} + C$, and energy consumption $r^{(3)} = -\sum_{i}a_i^2 - F_{\text{contact}} + C$. Here, $C = 1$ is the alive bonus, $F_{\text{contact}}$ is the contact force cost, and $a_i$ is the applied joint torque.

    \item \textbf{MO-Ant-2obj-v5:} A modified version of MO-Ant-v5 with identical state-action spaces but only two objectives. The energy and contact penalties are integrated into the velocity objectives: $x$-axis velocity $r^{(1)} = v_x -\sum_{i}a_i^2 - F_{\text{contact}} + C$ and $y$-axis velocity $r^{(2)} = v_y -\sum_{i}a_i^2 - F_{\text{contact}} + C$.
\end{itemize}

\paragraph{Reward Scalarization.} We utilize linear scalarization to transform the multi-objective problem into standard single-objective RL tasks \cite{roijers2013survey}. For an environment with reward vector $r_t \in \mathbb{R}^K$, we sample a weight vector $w \in \mathbb{R}^K$ such that $\sum w_i = 1$, and define the scalarized reward $\hat{r}_t = w^\top r_t$. For 2-objective environments like HalfCheetah, we perform uniform grid sampling over the weight simplex, selecting 10 equidistant weight vectors. For environments with more objectives, such as Ant, we sample 10 weight vectors from a symmetric Dirichlet distribution.

\paragraph{Policy Population.} For each sampled weight vector, we train an agent using PPO \cite{schulman2017proximal}. Our PPO algorithm is adapted from RSL-RL \cite{schwarke2025rsl}, and the hyperparamters used across all environments are reported in Table~\ref{tab:ppo_hyperparameters}.

\begin{table}[H]
    \centering
    \caption{PPO Hyperparameters.}
    \label{tab:ppo_hyperparameters}
    \begin{tabular}{lr}
        \toprule
        \textbf{Hyperparameter} & \textbf{Value} \\
        \midrule
        Parallel Environments & 16 \\
        Steps per Environment & 1024 \\
        Optimization Epochs & 10 \\
        Mini-batches per Epoch & 8 \\
        Learning Rate & $3 \times 10^{-4}$ \\
        Discount Factor & 0.99 \\
        GAE Parameter & 0.95 \\
        Clipping Parameter & 0.2 \\
        Value Loss Coefficient & 0.5 \\
        Entropy Coefficient & 0.0 \\
        Policy Network Architecture & MLP (256, 256) \\
        Value Network Architecture & MLP (256, 256) \\
        \bottomrule
    \end{tabular}
\end{table}

A key distinction between our approach and multi-objective RL (MORL) \cite{roijers2013survey} lies in the target distribution. While MORL typically focuses on identifying and approximating the Pareto Front \cite{pgmorl}, our objective is to model the entire policy space, including dominated regions. To this end, we retain intermediate policies throughout training, thereby enriching the policy population with a diverse set of behaviors that extends beyond near-optimal solutions. Dataset statistics are visualized in Appendix~\ref{appendix:imitation}. MORL algorithms \cite{pgmorl, basaklar2022pd} can be used to collect policies exhibiting more diverse behaviors, which we leave for future work.

\paragraph{Dataset Statistics.} Aggregating these checkpoints yields a population of approximately 500 distinct policies per environment. We execute each policy to collect 8,000 transition steps, resulting in a total dataset size of $N \approx 4 \times 10^6$ transitions per domain. Note that trajectories are collected using stochastic policies to boost the robustness of behavior cloning \cite{laskey2017dart}. The final dataset is structured as $\mathcal{D} = \{(\tau_n, R_n) \}_{n=1}^{N_{\text{total}}}$, where $R_n$ represents the attributes of interest associated with trajectory $\tau_n$. In our primary experiments, $R_n$ is defined as the vector of undiscounted cumulative rewards: $R_n = \sum_{t=1}^{T_{n}} r_t$, where $T_{n}$ denotes the length of trajectory $\tau_n$. However, our framework remains agnostic to the specific attribute definition and accommodates other trajectory-level characteristics, such as average velocity.

\begin{figure}[hbtp]
    \centering
    \includegraphics[width=0.55\columnwidth]{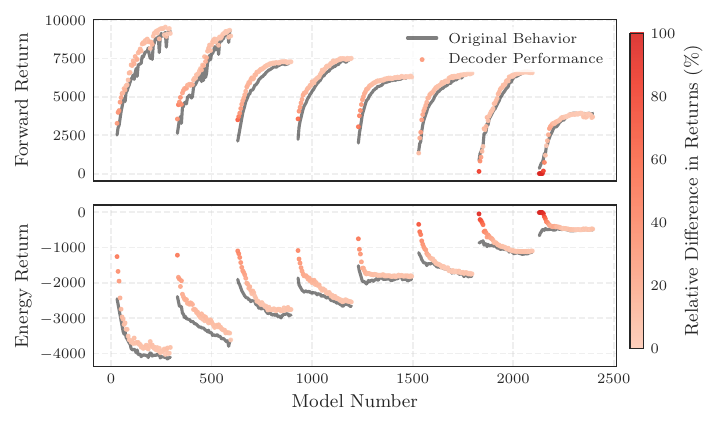}
    \caption{\textbf{Dataset composition and decoder performance in HalfCheetah.} We plot Forward Return (top) and Energy Return (bottom) over cumulative training checkpoints. Gray curves denote returns of trajectories collected by policy checkpoints. Colored dots represent the returns achieved by the decoder conditioned on inferred policy representations. Color intensity reflects the relative difference between the original and decoder-conditioned returns. The overall low intensity indicates that the decoder faithfully reproduces diverse behaviors.}
    \label{fig:sawtooth}
\end{figure}

To further illustrate the dataset composition, Figure~\ref{fig:sawtooth} visualizes the return statistics of the HalfCheetah policy population. The x-axis corresponds to the cumulative index of policy checkpoints saved throughout the entire PPO training process. We exclude scalarizations and checkpoints that induce stationary behaviors. For example, the optimal HalfCheetah policy becomes stationary when the energy reward coefficient is 1.0 and the forward reward coefficient is 0.0. The remaining 8 scalarizations are each trained for 300 steps, yielding a total of 2,400 checkpoints. We subsample every five checkpoints to form the training policy set. Solid gray curves show the returns of trajectories collected by the policy checkpoints.

\section{Implementation Details}\label{appendix:trainingdetails}
We execute the training procedure outlined in Algorithm~\ref{appendix:algorithm} using datasets collected from the environments of interest. All experiments are conducted on NVIDIA DGX B200 GPUs. To facilitate training, both the input observations and the regression targets (trajectory returns) are normalized to be zero mean and unit variance.

\subsection{Encoder Parametrization} 
We employ a BERT-style Transformer architecture \cite{devlin2019bert} for the policy encoder. Each context set $\mathcal{C}$ is treated as a sequence of $(s,a)$ tokens, with a learnable \texttt{[CLS]} token prepended. Each token is first processed by a two-layer MLP with hidden dimension 256 to produce a 32-dimensional embedding. These embeddings are then passed through two multi-head self-attention layers, each with 4 heads. The final output embedding of the \texttt{[CLS]} token is interpretted as the mean and log standard deviation of the posterior $p_\theta(\cdot \mid \mathcal{C})$, effectively representing policy as a latent distribution. We set the default embedding dimension to 32, resulting in a policy representation $\tilde{h} \in \mathbb{R}^{32}$. To improve generalization, we apply dropout with probability 0.1 to the inputs and self-attention layers. The projection heads $u^{(k)}$ reduce the dimensionality from 32 to 4, yielding per-task embeddings $z^{(k)} \in \mathbb{R}^{4}$ and projection matrices $W^{(k)} \in \mathbb{R}^{4\times32}$.

\subsection{Decoder Parametrization} 
We utilize a consistent decoder architecture across all experiments. The query state is first processed by a two-layer MLP with hidden dimension 256. The resulting feature vector is concatenated with the policy representation $\tilde{h}$ and further passed through a two-layer MLP with hidden dimension 256. The final layer outputs a vector of size $2 \times |\mathcal{A}|$, which is interpreted as the mean and diagonal log-standard deviation of the predicted action distribution.

\subsection{Hyperparameters}\label{appendix:hyperparameters}
The simultaneous minimization of the KL divergence, orthonormal regularization term and RNC loss creates a tension between regularization and geometric structuring. The KL term acts as a contraction force, constraining policy representations $\tilde{h}$ toward a standard Gaussian prior; while the RNC loss acts as an expanding force, driving embeddings apart to satisfy ordinal ranking constraints. We manage this trade-off via a KL annealing schedule ($\beta: 0.0 \to 0.05$), which allows geometric structure to emerge from the beginning. Empirically, the average representation norm converges to values between $2.0$ and $3.0$, with an average standard deviation of approximately $0.8$. $\zeta$ is chosen empirically to ensure effective regularization.

In practice, the similarity between embeddings is scaled by a temperature parameter $\tau_\text{sim}$, defining the similarity score as $s_{i,j}\smash{^{(k)}} := -\| z_i\smash{^{(k)}} - z_j\smash{^{(k)}} \|_2 / \tau_\text{sim}$. This scaling plays an important role. Setting $\tau_\text{sim} < 1$ sharpens the similarity distribution, which amplifies the gradients for the RNC loss and enforces stricter ranking in $\mathcal{Z}$-space even when $\tilde{h}$ remains close to the unit Gaussian. However, extremely low temperatures are avoided to prevent overfitting to the stochastic noise inherent in trajectory returns. A complete summary of hyperparameters used across all environments is provided in Table~\ref{tab:hyperparameters}.

\begin{table}[h]
\centering
\caption{Hyperparameters for Policy Representation Learning}
\label{tab:hyperparameters}
\begin{tabular}{lc|lc}
\toprule
\textbf{Parameter} & \textbf{Value} & \textbf{Parameter} & \textbf{Value} \\
\midrule
Context Length $|\mathcal{C}|$ & 32 & RNC Temp. ($\tau_\text{sim}$) & 0.5 \\
Learning Rate & $1 \times 10^{-3}$ & Rep. Learning Epochs ($E_{\text{rep}}$) & 200\\
$\alpha$ (Contrastive) & 1.0 & Rep. Batch Size ($I$) & 64 \\
$\zeta$ (Orthonormal Constraint) & 5.0 & Regressor Epochs ($E_{\text{reg}}$) & 100\\
$\beta_{\text{start}}$ (KL) & 0.0 & Regressor Batch Size ($J$) & 256 \\
$\beta_{\text{end}}$ (KL) & 0.05 & & \\
\bottomrule
\end{tabular}
\end{table}

\begin{algorithm}[H]
\caption{Structured Policy Representation Learning}
\label{appendix:algorithm}
\begin{algorithmic}[1]
\STATE \textbf{Input:} Dataset $\mathcal{D} = \{(\tau_n, R_n)\}_{n=1}^{N_{\text{total}}}$, Hyperparameters $\alpha, \beta, \zeta$.
\STATE \textbf{Initialize:} Encoder $g_\theta$, Decoder $\pi_\phi$, Linear Projection Heads $\{u_{\psi^{(k)}}\}_{k=1}^K$, Return Regressors $\{v_{\xi^{(k)}}\}_{k=1}^K$.

\vspace{0.5em}
\STATE \textcolor{brown}{$\blacktriangleright$ \textbf{Phase 1: Representation Learning}}
\vspace{0.2em}

\FOR{epoch $= 1$ \textbf{to} $E_{\text{rep}}$}
    \FOR{sampled batch $\{(\tau_i, R_i)\}_{i=1}^{I} \sim \mathcal{D}$}
        
        \STATE \textcolor{gray}{// Sampling context set to construct a two-view batch, essential for contrastive learning}
        \FOR{$i = 1$ \textbf{to} $I$}
            \STATE Sample two views: $\mathcal{C}_{2i-1}, \mathcal{C}_{2i} \sim \tau_i$.
            \STATE Sample query pairs: ${(s_i, a_i)} \sim \tau_i$.
        \ENDFOR
        \STATE Form augmented batch $\{(\mathcal{C}_i, R_{\lceil i/2 \rceil})\}_{i=1}^{2I}$.

        \vspace{0.5em}
        \STATE \textcolor{gray}{// Forward Pass}
        \STATE Compute posteriors: $q_\theta(\cdot \mid \mathcal{C}_i) = g_\theta(\mathcal{C}_i)$ for $i=1 \dots I$.
        \STATE Sample policy representations: $\tilde{h}_i \sim q_\theta(\cdot \mid \mathcal{C}_i)$ via reparameterization for $i=1 \dots I$.
        \STATE Project task-specific embeddings: $z_i^{(k)} = u_{\psi^{(k)}}(\tilde{h}_i)$ for $m=1 \dots M$ and $i=1 \dots I$.

        \vspace{0.5em}
        \STATE \textcolor{gray}{// Compute Losses \& Update}
        \STATE $\ell_{\pi}(\theta, \phi) \leftarrow \frac{1}{2I} \sum_{i=1}^{2I}\left[ - \log \pi_\phi(a_i \mid s_i, \tilde{h}_i) + \beta  D_{\mathrm{KL}} \big( q_\theta(\cdot \mid \mathcal{C}_i) \;\|\; \mathcal{N}(0, I) \right]$
        \STATE $\ell_v^0(\theta, \psi^{(k)}) \leftarrow -\frac{1}{2I(2I-1)}
        \sum_{i=1}^{2I}
        \sum_{\substack{j=1 \\ j\neq i}}^{2I} \log\frac{\exp(s^{(k)}_{i,j})}{\sum_{l \in \mathcal{F}_{i,j}^{(k)}} \exp(s^{(k)}_{i,l})}$
        \STATE $\ell_v^1(\psi^{(k)}) \leftarrow \| W^{(k)} (W^{(k)})^\top - I_{d_{z_k}} \|_F^2$
        \STATE $\ell_v(\theta, \psi) \leftarrow \frac{1}{K}\sum_{k=1}^{K} \left[ \alpha\ell_v^0(\theta, \psi^{(k)}) + \zeta \ell_v^1(\psi^{(k)}) \right]$
        \STATE $\theta, \phi, \psi \leftarrow \text{AdamW}\big( \nabla \ell_{\pi}(\theta, \phi) + \nabla \ell_v(\theta, \psi)\big)$
    \ENDFOR
\ENDFOR

\vspace{0.8em}
\STATE \textcolor{brown}{$\blacktriangleright$ \textbf{Phase 2: Return Regressor Training}}
\vspace{0.2em}
\STATE Freeze encoder $g_\theta$ and projection heads $\{u_{\psi^{(k)}}\}_{k=1}^K$.
\FOR{epoch $= 1$ \textbf{to} $E_{\text{reg}}$}
    \FOR{sampled batch $\{(\tau_j, R_j)\}_{j=1}^{J} \sim \mathcal{D}$}
        \STATE \textcolor{gray}{// Batch Encoding}
        \STATE Sample contexts: $\mathcal{C}_j \sim \tau_j$ for $j=1 \dots J$.
        \STATE Sample representations: $\tilde{h}_j \sim q_\theta(\cdot \mid \mathcal{C}_j)$ for $j=1 \dots J$.

        \vspace{0.5em}        
        \STATE \textcolor{gray}{// Regression Update}
        \STATE $\ell_v^{2}(\xi) \leftarrow \frac{1}{J} \sum_{j=1}^{J} \sum_{k=1}^K \left[ \left\| v_{\xi^{(k)}}\big( u_{\psi^{(k)}}(\tilde{h}_j) \big) - R_j^{(k)} \right\|_2^2 \right]$
        \STATE $\xi \leftarrow \text{AdamW}\big(\nabla \ell_v^{2}(\xi)\big)$
    \ENDFOR
\ENDFOR
\end{algorithmic}
\end{algorithm}

\section{Additional Experiments}
\label{appendix:experiments}

\subsection{Imitation of Diverse Behavior}
\label{appendix:imitation}
We provide additional imitation results for other Multi-Objective MuJoCo environments. Similar to Figure~\ref{fig:side_by_side_comparison}, each left subplot visualizes the distribution of collected policies in the multi-objective return landscape. Each right subplot illustrates the quality of imitation learning, quantified as the relative difference in returns between the original policy checkpoints and the trajectories collected by the decoder conditioned on the inferred policy representation. For environments with three objectives, we visualize the policy distribution using ternary plots projected onto a 2-simplex.  We observe decent coverage of the objective space with the collected training dataset, and the decoder achieves trajectory returns similar to the original policy checkpoints. This indicates that the proposed training algorithm yields an encoder capable of generating distinct representations and a decoder capable of reconstructing diverse behaviors.

\begin{figure}[h!]    
    \centering
    \begin{subfigure}[b]{0.48\textwidth}
        \includegraphics[width=\textwidth]{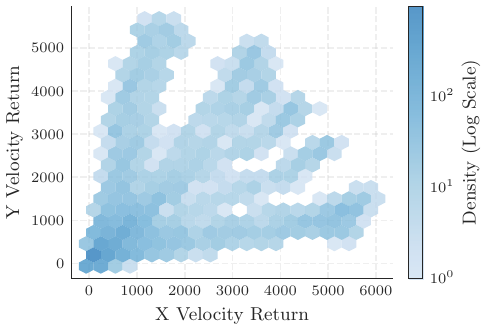}
        \caption{Dataset Diversity.}
    \end{subfigure}
    \hfill
    \begin{subfigure}[b]{0.48\textwidth}
        \includegraphics[width=\textwidth]{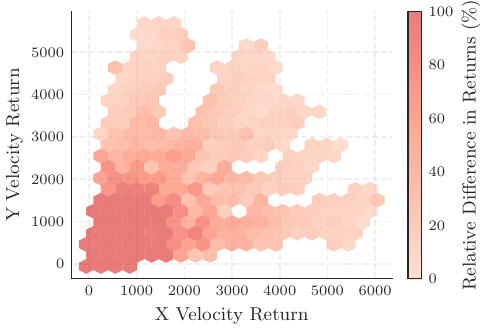}
        \caption{Imitation Learning Performance.}
    \end{subfigure}
    \caption{\textbf{Behavioral diversity and imitation performance for MO-Ant-2obj-v5.}}
    \label{fig:ant2obj_compare}
\end{figure}

\begin{figure}[h!]    
    \centering
    \begin{subfigure}[b]{0.48\textwidth}
        \includegraphics[width=\textwidth]{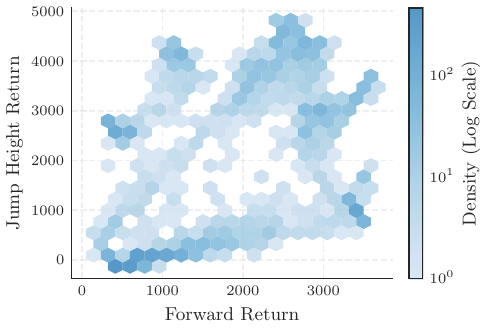}
        \caption{Dataset Diversity.}
    \end{subfigure}
    \hfill
    \begin{subfigure}[b]{0.48\textwidth}
        \includegraphics[width=\textwidth]{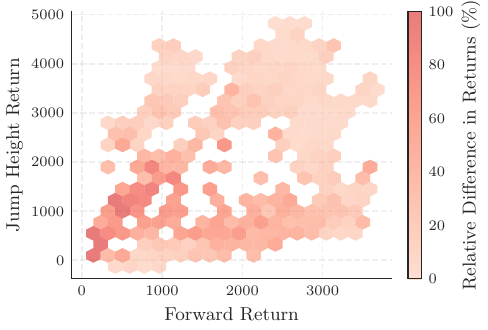}
        \caption{Imitation Learning Performance.}
    \end{subfigure}
    \caption{\textbf{Behavioral diversity and imitation performance for MO-Hopper-2obj-v5.}}
    \label{fig:hopper2obj_compare}
\end{figure}

\begin{figure}[h!]    
    \centering
    \begin{subfigure}[b]{0.48\textwidth}
        \includegraphics[width=\textwidth]{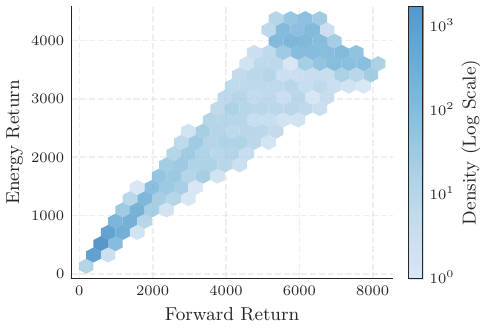}
        \caption{Dataset Diversity.}
    \end{subfigure}
    \hfill
    \begin{subfigure}[b]{0.48\textwidth}
        \includegraphics[width=\textwidth]{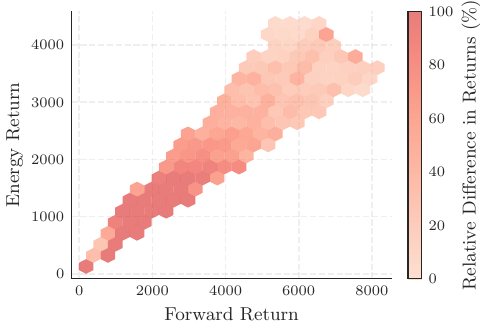}
        \caption{Imitation Learning Performance.}
    \end{subfigure}
    \caption{\textbf{Behavioral diversity and imitation performance for MO-Humanoid-v5.}}
    \label{fig:humanoid_compare}
\end{figure}

\begin{figure}[h!]    
    \centering
    \begin{subfigure}[b]{0.48\textwidth}
        \includegraphics[width=\textwidth]{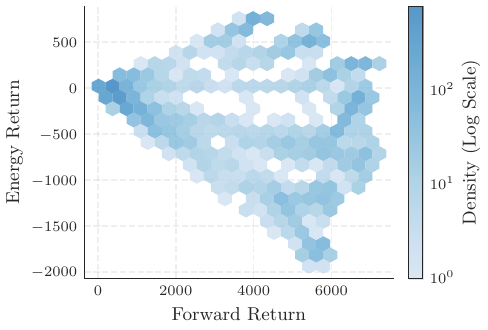}
        \caption{Dataset Diversity.}
    \end{subfigure}
    \hfill
    \begin{subfigure}[b]{0.48\textwidth}
        \includegraphics[width=\textwidth]{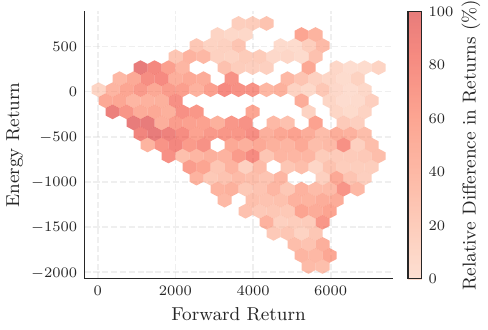}
        \caption{Imitation Learning Performance.}
    \end{subfigure}
    \caption{\textbf{Behavioral diversity and imitation performance for MO-Walker2d-v5.}}
    \label{fig:walker_compare}
\end{figure}

\begin{figure}[h!]    
    \centering
    \begin{subfigure}[b]{0.48\textwidth}
        \includegraphics[width=\textwidth]{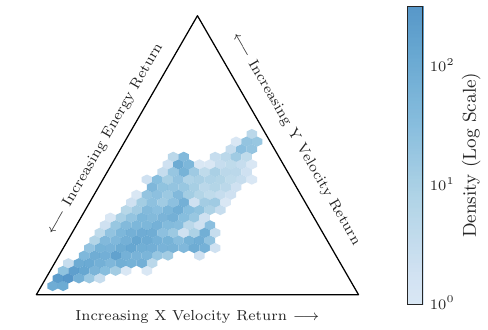}
        \caption{Dataset Diversity.}
    \end{subfigure}
    \hfill
    \begin{subfigure}[b]{0.48\textwidth}
        \includegraphics[width=\textwidth]{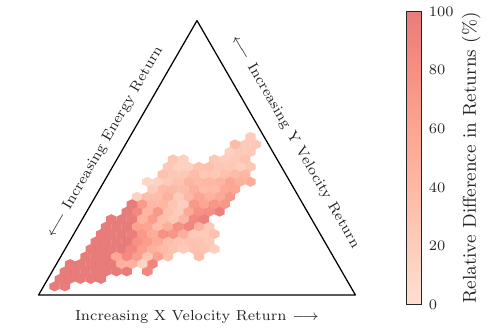}
        \caption{Imitation Learning Performance.}
    \end{subfigure}
    \caption{\textbf{Behavioral diversity and imitation performance for MO-Ant-v5.}}
    \label{fig:ant_compare}
\end{figure}

\begin{figure}[H]    
    \centering
    \begin{subfigure}[b]{0.48\textwidth}
        \includegraphics[width=\textwidth]{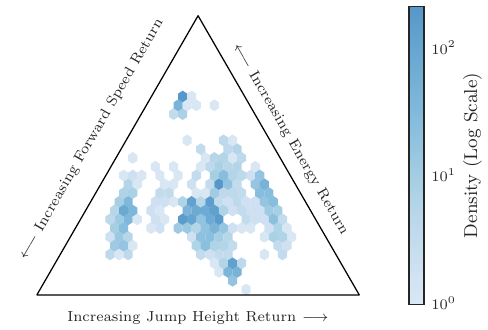}
        \caption{Dataset Diversity.}
    \end{subfigure}
    \hfill
    \begin{subfigure}[b]{0.48\textwidth}
        \includegraphics[width=\textwidth]{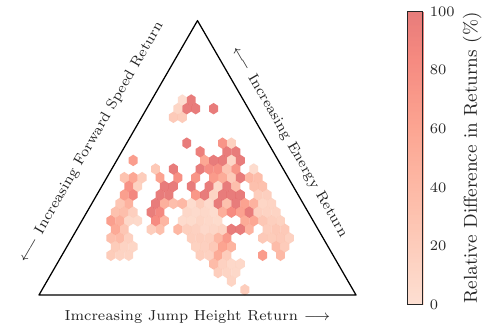}
        \caption{Imitation Learning Performance.}
    \end{subfigure}
    \caption{\textbf{Behavioral diversity and imitation performance for MO-Hopper-v5.}}
    \label{fig:hopper_compare}
\end{figure}

\FloatBarrier
\subsection{Geometry of Policy Representations}
To validate the effectiveness of the proposed algorithm, we visualize the learned policy representations for MO-MuJoCo environments. From dataset for each environment, we randomly sample a set of 2000 context-return pairs $\{(\mathcal{C}_i, R_i)\}_{i=1}^{2000}$ along with a policy representation for each $\tilde{h}_i \sim q_\theta (\cdot \mid \mathcal{C}_i)$. The resulting representations are projected using UMAP \cite{mcinnes2018umap} and colored according to individual return objectives. The UMAP parameters are set to their default values except for the number of neighbors. We choose \texttt{n\_neighbors} = 100 to prioritize the preservation of global structure over local clustering. The plots are augmented with isocontours estimated from the return values to demonstrate that the geometry of the latent space aligns with the ordinal structure of behavioral attributes.

Note that proximity in the UMAP visualization reflects local neighborhood structures rather than global Euclidean distances, which partially explains the separation between clusters in Figures~\ref{fig:hopper2obj_geometry} and~\ref{fig:hopper_geometry}. We anticipate that increasing the KL penalty and collecting a more diverse set of policies to cover the behavioral landscape would result in a more compact latent space.

\begin{figure}[h]    
    \centering
    \begin{subfigure}[b]{0.48\textwidth}
        \includegraphics[width=\textwidth]{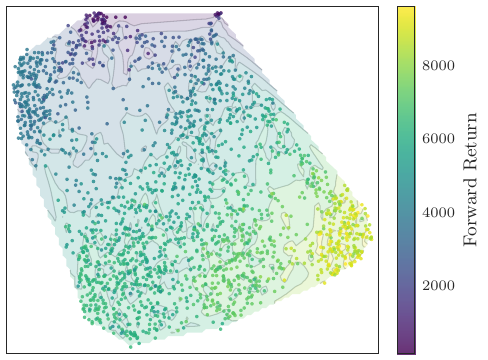}
        \caption{Forward Return Landscape.}
    \end{subfigure}
    \hfill
    \begin{subfigure}[b]{0.48\textwidth}
        \includegraphics[width=\textwidth]{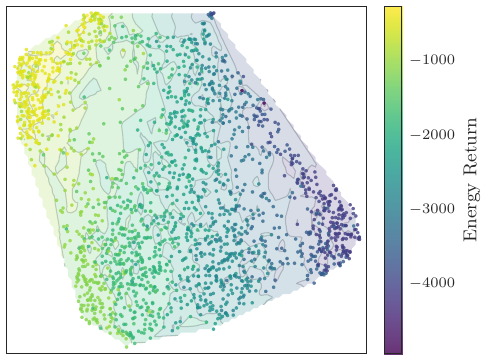}
        \caption{Energy Return Landscape.}
    \end{subfigure}
    \caption{\textbf{Geometry of policy representations for MO-HalfCheetah-v5}}
    \label{fig:cheetah_geometry}
\end{figure}

\begin{figure}[!h]    
    \centering
    \begin{subfigure}[b]{0.48\textwidth}
        \includegraphics[width=\textwidth]{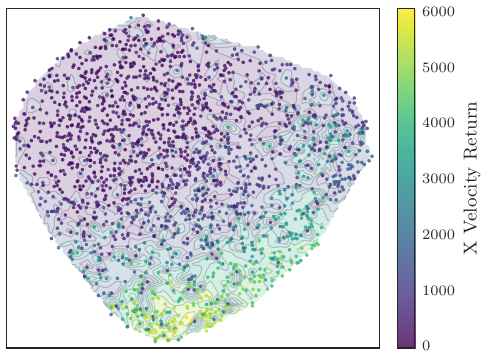}
        \caption{X Velocity Return Landscape.}
    \end{subfigure}
    \hfill
    \begin{subfigure}[b]{0.48\textwidth}
        \includegraphics[width=\textwidth]{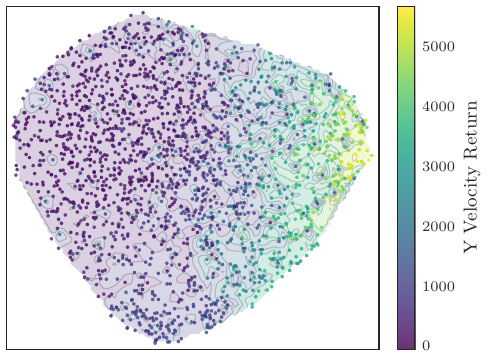}
        \caption{Y Velocity Return Landscape.}
    \end{subfigure}
    \caption{\textbf{Geometry of policy representations for MO-Ant-2obj-v5}}
    \label{fig:ant2obj_geometry}
\end{figure}

\begin{figure}[!h]    
    \centering
    \begin{subfigure}[b]{0.48\textwidth}
        \includegraphics[width=\textwidth]{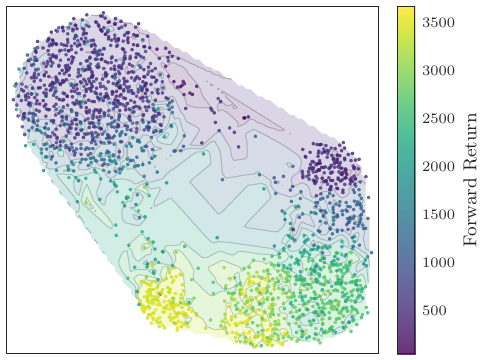}
        \caption{Forward Return Landscape.}
    \end{subfigure}
    \hfill
    \begin{subfigure}[b]{0.48\textwidth}
        \includegraphics[width=\textwidth]{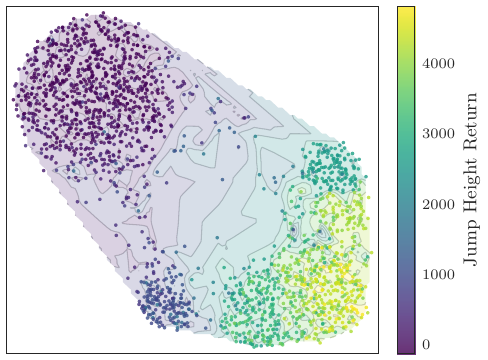}
        \caption{Jump Height Return Landscape.}
    \end{subfigure}
    \caption{\textbf{Geometry of policy representations for MO-Hopper-2obj-v5}}
    \label{fig:hopper2obj_geometry}
\end{figure}

\begin{figure}[h]    
    \centering
    \begin{subfigure}[b]{0.48\textwidth}
        \includegraphics[width=\textwidth]{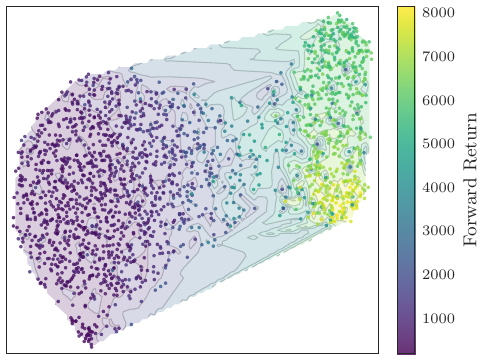}
        \caption{Forward Return Landscape.}
    \end{subfigure}
    \hfill
    \begin{subfigure}[b]{0.48\textwidth}
        \includegraphics[width=\textwidth]{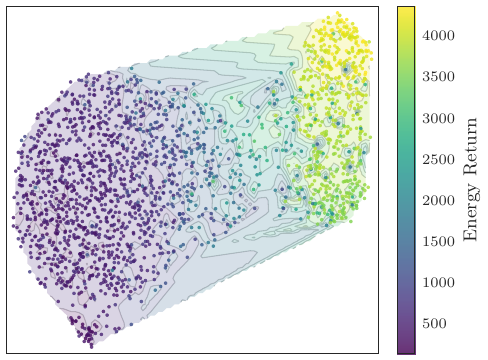}
        \caption{Energy Return Landscape.}
    \end{subfigure}
    \caption{\textbf{Geometry of policy representations for MO-Humanoid-v5}}
    \label{fig:humanoid_geometry}
\end{figure}


\begin{figure}[h]    
    \centering
    \begin{subfigure}[b]{0.32\textwidth}
        \includegraphics[width=\textwidth]{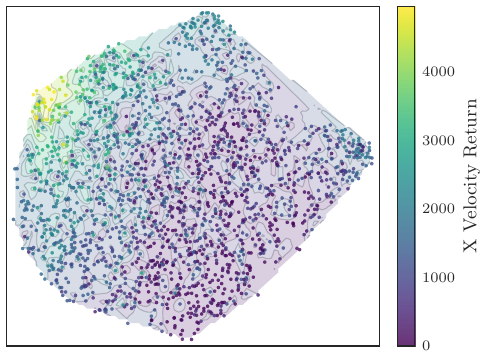}
        \caption{X Velocity Return Landscape.}
    \end{subfigure}
    \hfill
    \begin{subfigure}[b]{0.32\textwidth}
        \includegraphics[width=\textwidth]{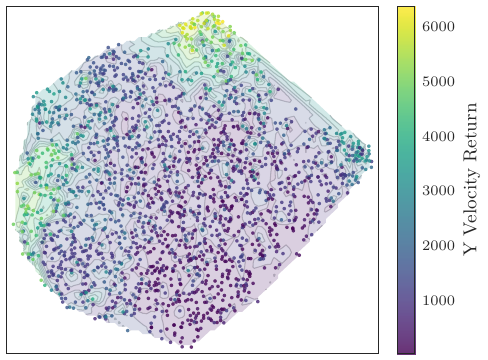}
        \caption{Y Velocity Return Landscape.}
    \end{subfigure}
    \hfill
    \begin{subfigure}[b]{0.32\textwidth}
        \includegraphics[width=\textwidth]{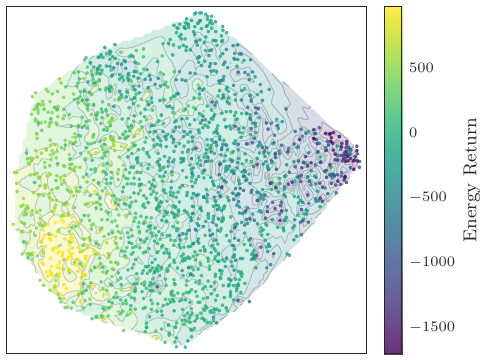}
        \caption{Energy Return Landscape.}
    \end{subfigure}
    \caption{\textbf{Geometry of policy representations for MO-Ant-v5}}
    \label{fig:ant_geometry}
\end{figure}

\begin{figure}[H]    
    \centering
    \begin{subfigure}[b]{0.32\textwidth}
        \includegraphics[width=\textwidth]{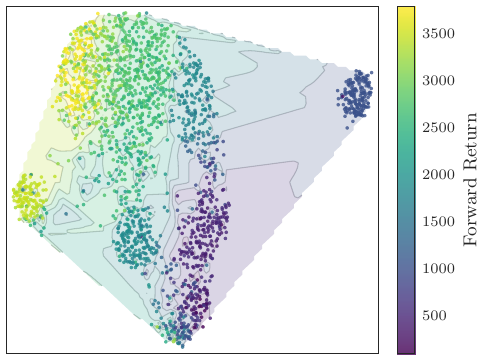}
        \caption{Forward Return Landscape.}
    \end{subfigure}
    \hfill
    \begin{subfigure}[b]{0.32\textwidth}
        \includegraphics[width=\textwidth]{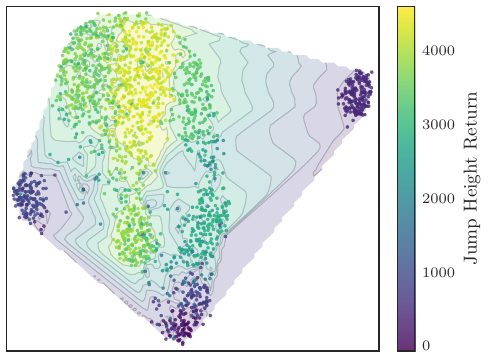}
        \caption{Jump Height Return Landscape.}
    \end{subfigure}
    \hfill
    \begin{subfigure}[b]{0.32\textwidth}
        \includegraphics[width=\textwidth]{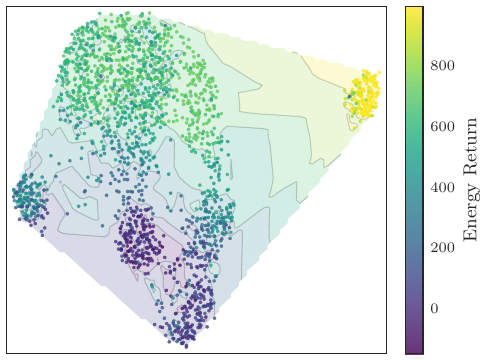}
        \caption{Energy Return Landscape.}
    \end{subfigure}
    \caption{\textbf{Geometry of policy representations for MO-Hopper-v5}}
    \label{fig:hopper_geometry}
\end{figure}

\FloatBarrier
\subsection{Additional Qualitative Results for Behavior Synthesis with Constraint}
We provide two additional qualitative results for steering in HalfCheetah. Figure~\ref{fig:steering_appendix1} illustrates a trajectory initialized at $v_{\text{init}}^{(1)} \approx 2000$ and $v_{\text{init}}^{(2)} \approx -2200$, which traverses exactly along the constraint boundary to maximize forward return while maintaining feasibility, with $v_g^{(1)} = 7000$ and $v_c^{(2)} = -2200$. Figure~\ref{fig:steering_appendix2} shows an initialization at an aggressive policy with $v_{\text{init}}^{(1)} \approx 8000$ and $v_{\text{init}}^{(2)} \approx -4000$, which traverses across the manifold to reach a conservative target. with $v_g^{(1)} = 4500$ and $v_c^{(2)} = -500$

\begin{figure}[H]
    \centering
    \begin{subfigure}[b]{0.48\textwidth}
        \includegraphics[width=\textwidth]{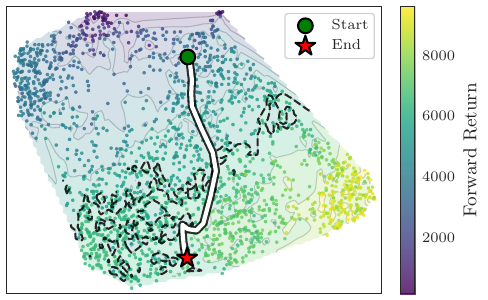}
    \end{subfigure}
    \hfill
    \begin{subfigure}[b]{0.48\textwidth}
        \includegraphics[width=\textwidth]{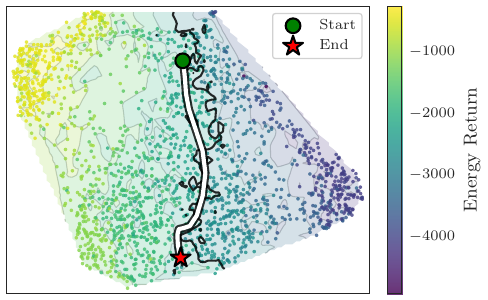}
    \end{subfigure}

    \caption{(a) Trajectory progresses the manifold toward the iso-contour corresponding to the target return $v_g^{(1)}=7000$ (dotted line). (b) Trajectory satisfies energy constraint $v^{(2)} \ge v_c^{(2)} = -2200$, traversing along the feasible boundary (solid line) via primal-dual updates.}
    \label{fig:steering_appendix1}
\end{figure}
\begin{figure}[H]
    \centering
    \begin{subfigure}[b]{0.48\textwidth}
        \includegraphics[width=\textwidth]{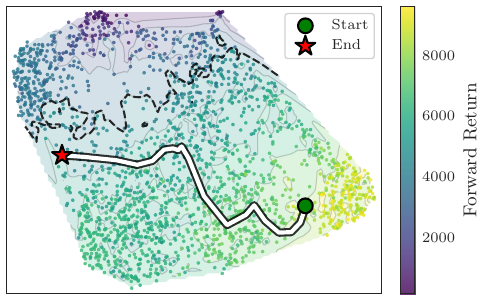}
    \end{subfigure}
    \hfill
    \begin{subfigure}[b]{0.48\textwidth}
        \includegraphics[width=\textwidth]{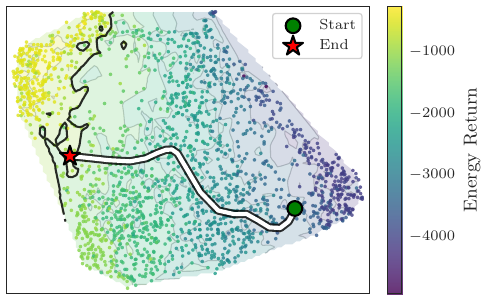}
    \end{subfigure}

    \caption{(a) Trajectory progresses the manifold toward the iso-contour corresponding to the target return $v_g^{(1)}=4500$ (dotted line). (b) Trajectory satisfies energy constraint $v^{(2)} \ge v_c^{(2)} = -500$.}
    \label{fig:steering_appendix2}
\end{figure}

\section{Definitions}

\subsection{Control functional integration.}
\label{app:control-functional}

To approximate the policy representation
\(
h_\pi = \mathbb{E}_{d_\pi}[f(s,a)],
\)
we adopt the control functional framework of \citet{oates2017control}, which addresses the general problem of estimating expectations of the form
\begin{equation}
\mu(f) := \int_{\Omega} f(x)\, p(x)\, \mathrm{d}x,
\label{eq:oates_target}
\end{equation}
where $\Omega \subset \mathbb{R}^d$ is a compact domain, $p$ is a probability density on $\Omega$, and $f$ is a test function.
In our setting, $x=(s,a)$ and $p=d_\pi$.
The control functional approach improves upon standard Monte Carlo estimation by constructing an RKHS surrogate $\hat f$ of $f$ whose expectation under $p$ is available in closed form.
The estimator is obtained by decomposing
\begin{equation}
f(x)
=
\hat f(x)
+
\bigl(f(x) - \hat f(x)\bigr),
\label{eq:cf_decomposition}
\end{equation}
where $\hat f$ is the RKHS control functional.
Following \citet[Sec.~2.1.1]{oates2017control}, the sample set is partitioned to ensure that the surrogate $\hat f$ is constructed independently of the samples used to estimate the residual term.
Specifically, we partition
\(
\mathcal D_0 = \{x_i\}_{i=1}^m
\)
and
\(
\mathcal D_1 = \{x_i\}_{i=m+1}^N.
\)
The control functional $\hat f$ is obtained as the regularised least--squares solution
\begin{equation}
\arg\min_{g \in \mathcal H_\kappa}
\sum_{i=1}^m \| f(x_i) - g(x_i) \|^2
+
\lambda \| g \|_{\mathcal H_\kappa}^2.
\label{eq:cf_ridge}
\end{equation}

We define
\[
f_0 := \bigl[f(x_1),\ldots,f(x_m)\bigr]^\top,
\qquad
f_1 := \bigl[f(x_{m+1}),\ldots,f(x_N)\bigr]^\top,
\qquad
\mathbf 1 := [1,\ldots,1]^\top.
\]

Let \(K_0 \in \mathbb R^{m\times m}\) be the kernel matrix with entries
\(
(K_0)_{i,j} := k_0(x_i,x_j),
\)
and let \(K_{1,0}\in\mathbb R^{(N-m)\times m}\) be the cross-kernel matrix with entries
\(
(K_{1,0})_{i,j} := k_0(x_{m+i},x_j),
\)
where $k_0$ denotes the Stein control functional kernel associated with the base kernel $\kappa$
and the density $p$ \citep[Sec.~2.1.1]{oates2017control}.
Then Lemma~3 in \citet{oates2017control} gives the following closed-form expression for the
control functional estimator
\begin{equation}
\hat\mu(\mathcal D_0,\mathcal D_1; f)
=
\frac{1}{N-m}\,\mathbf 1^\top\!\bigl(f_1 - \hat f_1\bigr)
\;+\;
\frac{\mathbf 1^\top (K_0+\lambda m I)^{-1} f_0}{\mathbf 1^\top (K_0+\lambda m I)^{-1}\mathbf 1},
\label{eq:oates_lemma3_adapted}
\end{equation}
where
\begin{equation}
\hat f_1
:=
K_{1,0}(K_0+\lambda m I)^{-1} f_0
+
\Bigl(\mathbf 1 - K_{1,0}(K_0+\lambda m I)^{-1}\mathbf 1\Bigr)
\,
\frac{\mathbf 1^\top (K_0+\lambda m I)^{-1} f_0}{\mathbf 1^\top (K_0+\lambda m I)^{-1}\mathbf 1}.
\label{eq:oates_fhat1_adapted}
\end{equation}

As noted in \citet[Remark~5]{oates2017control}, the estimator
\(\hat\mu(\mathcal D_0,\mathcal D_1; f)\) is a weighted combination of the function values
\(
[f_0^\top,f_1^\top]^\top
\),
with weights that sum to one and are independent of the particular test function $f$.

\subsection{RNC $\delta$ ordering}
\label{app:delta_ordering}

We recall the formal notion of $\delta$-ordered feature embeddings, as used in the theoretical analysis of RNC \citep{zha2023rank}, which characterizes how continuous target values induce a structured ordering in representation space.
Let $\mathcal{Z}$ denote the input space, and let $R$ be a continuous scalar target. 
Given a finite dataset $\{z_i\}_{i=1}^I$, with associated $\{R_i\}_{i=1}^I$. 
For any two samples $z_i,z_j\in\mathcal{Z}$, define the target distance
\[
d_{i,j} := |R_i - R_j|.
\]

Let's also define the similarity score
\[
s_{i,j} := -\|z_i - z_j\|_2.
\]

Given $\delta>0$, the representations $\{z_i\}_{i=1}^I$ are said to be $\delta$-ordered if, for all $k$, any $i\in[I]$, and any $j,\ell\in[I]\setminus\{i\}$, the following conditions hold
\[
\begin{cases}
s_{i,j} > s_{i,\ell} + \dfrac{1}{\delta}
& \text{if } d_{i,j} < d_{i,\ell}, \\[8pt]
\bigl| s_{i,j} - s_{i,\ell} \bigr| < \delta
& \text{if } d_{i,j} = d_{i,\ell}, \\[8pt]
s_{i,j} < s_{i,\ell} - \dfrac{1}{\delta}
& \text{if } d_{i,j} > d_{i,\ell}.
\end{cases}
\]

That is, for a fixed anchor $z_i$, samples that are closer in target space have strictly higher similarity to the anchor than samples that are farther away, with a margin controlled by $\delta$. When two samples are equidistant in target space, their similarities to the anchor are required to be approximately equal, up to tolerance $\delta$.

This notion enforces a relative geometric ordering in representation space that mirrors the ordering induced by the target values, without requiring similarities to be proportional to target differences. Optimizing the RNC objective encourages representations that satisfy this $\delta$-ordering property \citep[Theorem 3]{zha2023rank}.


\end{document}